\documentclass[nohyperref]{article}

\usepackage{microtype}
\usepackage{graphicx}
\usepackage{subfigure}
\usepackage{booktabs} 

\usepackage{stmaryrd} 
\usepackage{diagbox}

\usepackage{hyperref}

\usepackage[accepted]{icml2022}

\usepackage{amsmath}
\usepackage{amssymb}
\usepackage{mathtools}
\usepackage{amsthm}

\usepackage[capitalize,noabbrev]{cleveref}

\usepackage{bbm}
\usepackage{amsfonts}
\usepackage{amsmath}
\usepackage{amssymb}
\usepackage{longtable}
\usepackage{tabularx}
\usepackage{multirow}
\usepackage{graphicx}
\usepackage{xcolor}
\usepackage{algorithm}
\usepackage{algorithmic}
\usepackage{caption}
\usepackage{rotating}
\usepackage{xspace}

\newcommand{\inner}[1]{\left\langle#1\right\rangle}

\def\X{\mathcal{X}}

\def\R{\mathbb{R}}

\newcommand{\norm}[1]{\left\|#1\right\|}

\def\argmax{\mathop{\rm arg\,max}\limits}
\def\argmin{\mathop{\rm arg\,min}\limits}
\def\minop{\mathop{\rm min}\limits}
\def\maxop{\mathop{\rm max}\limits}
\def\sign{\mathop{\rm sign}\limits}

\def\min{\mathop{\rm min}\nolimits}
\def\max{\mathop{\rm max}\nolimits}

\newcommand{\ours}{\textrm{PNPC}\xspace}

\theoremstyle{plain}
\newtheorem{theorem}{Theorem}[section]
\newtheorem{proposition}[theorem]{Proposition}
\newtheorem{lemma}[theorem]{Lemma}

\theoremstyle{definition}
\newtheorem{definition}[theorem]{Definition}

\theoremstyle{remark}

\newcommand{\vv}[1]{{\color{black}{#1}}}

\usepackage[textsize=tiny]{todonotes}

\icmltitlerunning{Provably Adversarially Robust Nearest Prototype Classifiers}

\begin{document}

\twocolumn[
\icmltitle{Provably Adversarially Robust Nearest Prototype Classifiers}



\icmlsetsymbol{equal}{*}

\begin{icmlauthorlist}
\icmlauthor{Václav Voráček}{yyy}
\icmlauthor{Matthias Hein}{yyy}
\end{icmlauthorlist}

\icmlaffiliation{yyy}{University of T{\"u}bingen, Germany}

\icmlcorrespondingauthor{Václav Voráček}{vaclav.voracek@uni-tuebingen.de}
\icmlcorrespondingauthor{Matthias Hein}{matthias.hein@uni-tuebingen.de}

\icmlkeywords{Machine Learning, ICML}

\vskip 0.3in
]



\printAffiliationsAndNotice{}  

\begin{abstract}

Nearest prototype classifiers (NPCs) assign to each input point the label of the nearest prototype with respect to a chosen distance metric. A direct advantage of NPCs is that the decisions are interpretable. Previous work could provide lower bounds on the minimal adversarial perturbation in the $\ell_p$-threat model when using the same $\ell_p$-distance for the NPCs. In this paper we provide a complete discussion on the complexity when using $\ell_p$-distances for decision and $\ell_q$-threat models for certification for $p,q \in \{1,2,\infty\}$. In particular we provide scalable algorithms for the \emph{exact} computation of the minimal adversarial perturbation when using $\ell_2$-distance and improved lower bounds in other cases. Using efficient improved lower bounds we train our \vv{\textbf{P}rovably adversarially robust \textbf{NPC} (\ours)}, for MNIST which have better $\ell_2$-robustness guarantees than neural networks. Additionally, we show up to our knowledge the first certification results w.r.t. to the LPIPS perceptual metric which has been argued to be a more realistic threat model  for image classification than $\ell_p$-balls. Our \ours has on CIFAR10 higher certified robust accuracy than the empirical robust accuracy reported in \cite{laidlaw2021perceptual}. The code is available in our~\href{https://github.com/vvoracek/Provably-Adversarially-Robust-Nearest-Prototype-Classifiers}{repository}.



\end{abstract}


\section{Introduction}
\label{sec:intro}

The vulnerability of neural networks against adversarial manipulations \citep{SzeEtAl2014,GooShlSze2015} is a major problem for their real world deployment in safety critical systems such as autonomous driving and medical applications. However, the problem is not restricted to neural networks as it has been shown that basically all machine learning algorithms are vulnerable to adversarial perturbations e.g. nearest neighbor methods (NN) \cite{wang2018analyzing}, kernel SVMs \cite{XuCarMan2009, biggio2013evasion, russu2016secure, HeiAnd2017}, decision trees \cite{papernot2016transferability, bertsimas2018robust, chen2019robust,andriushchenko2019provably}. In the area of neural networks this lead to an arm's race between novel empirical defenses and attacks and even initially promising defenses were broken later on \cite{AthEtAl2018}. This still happens for papers published at top machine learning conferences \cite{tramer2020adaptive,croce2020provable}  despite more reliable attacks for adversarial robustness evaluation \cite{croce2020reliable} and guidelines \cite{carlini2019evaluating} being available. 

Thus classifiers with provable adversarial robustness guarantees are highly desirable. For neural networks computation of the exact minimal perturbation turns out to be restricted to very small networks \cite{TjeTed2017}. Instead one derives either deterministic \cite{HeiAnd2017,WonKol2018,GowEtAl18,MirGehVec2018,zhang2019stable,Lee2020Lipschitz,huang2021training,leino2021globallyrobust} or probabilistic guarantees 
\cite{CohenARXIV2019,jeong2021smoothmix} on the robust accuracy. We refer to \cite{li2020sokcertified} for a recent overview. While provable adversarial robustness has been studied extensively for neural networks, the literature for standard classifiers is  scarce, e.g. decision trees \cite{bertsimas2018robust},
boosted decision stumps and trees \cite{chen2019robust,andriushchenko2019provably}, and
nearest neighbour \cite{wang2018analyzing, wang2019evaluating} and nearest prototype classifiers (NPC) \cite{Saralajew2020fast}. NPC are also known as \emph{Learning Vector Quantization (LVQ)}, see \cite{Kohonen1995}, and are directly interpretable, can be used for all data where a distance function is available and have the advantage compared to a nearest neighbour classifier that the prototypes can be learned and thus they are more efficient and achieve typically better generalization performance. Moreover, NPC have a maximum margin nature \cite{NIPS2002_bbaa9d6a} and  \cite{Saralajew2020fast} showed recently how to derive lower bounds on the minimal adversarial perturbation which in turn yield lower bounds on the robust accuracy. \cite{wang2019evaluating} have shown how to compute the minimal adversarial perturbation for nearest neighbor classifiers using the $\ell_2$-distance which applies to NPC as well.

\textbf{Contributions:} we show that the results of \cite{Saralajew2020fast} can be improved in various ways leading to our \ours which perform better both in clean and robust accuracy.

A) We generalize the lower bounds on the minimal adversarial perturbation
\cite{Saralajew2020fast} provided for distances induced by semi-norms to general semi-metrics, thus improving significantly over standard $\ell_p$-based certification. \vv{
The original proof of \cite{Saralajew2020fast} used the absolute homogenity of semi-norms; thus, it do not generalize to semi-metrics.}

B) For NPC using the $\ell_2$-distance we show that the lower bounds of \cite{wang2019evaluating} can be quickly evaluated so that training with them is feasible and show that these bounds improve the ones of \cite{Saralajew2020fast}. Moreover, we improve the certification of \cite{wang2019evaluating} by integrating that the domain in image classification is $[0,1]^d$. For MNIST our $\ell_2$-\ours has the best $\ell_2$-robust accuracy even outperforming randomized smoothing for large radii.
Moreover, we show how to certify exactly $\ell_1$- and $\ell_\infty$-robustness for $\ell_2$-NPC and in this way can certify multiple-norm robustness and show that our $\ell_2$-\ours outperforms 
the multiple-norm robustness guarantees of \cite{croce2020provable}.

C) For the $\ell_1$-and $\ell_\infty$-NPC we provide novel lower bounds and analyze their complexity. For $\ell_\infty$-NPCs we thus improve over the bounds given in \cite{Saralajew2020fast}.

D) As the $\ell_p$-distances are not suited for image classification tasks, we use a neural perceptual metric (LPIPS) \cite{zhang2018unreasonable} as a semi-metric for the NPC and provide robustness guarantees in the perceptual metric. We improve both in terms of clean and certified robust accuracy over the clean and empirical robust accuracy of the adversarially trained ResNet 50 of  \cite{laidlaw2021perceptual} 

\section{Provably Robust NPC Classifiers}\label{sec:proto}
\vv{
Nearest prototype classifiers require for a given input space $\X$ only a (semi)-metric. To compare with previous work, we introduce also a (semi)-norm, which requires a vector-space structure; thus, assuming the existence of a norm is a stronger assumption than the assumption of the existence of a metric.

\begin{definition}
A mapping $d:\X \times \X \rightarrow \R$ is a semi-metric if the following properties holds for any $x,y,z \in \X$:
\begin{itemize}
    \item $d(x,y) \geq 0$ (non-negativity)
    \item $d(x,y) = d(y,x)$ (symmetry)
    \item $d(x,y) \leq d(x,z) + d(z,y)$ (triangle inequality)
\end{itemize}
If we further require that $d(x,y) = 0 \implies x=y$, then the semi-metric becomes a metric.
\end{definition}

\begin{definition}
A mapping $\norm{\cdot}:\X \rightarrow \R$ is a semi-norm if the following properties holds for any $x,y\in \X$, $\alpha \in \R$:
\begin{itemize}
    \item $\norm{x} \geq 0$ (non-negativity)
    \item $\norm{\alpha x} = |\alpha|\norm{x}$ (absolute homogeneity)
    \item $\norm{x+y} \leq \norm{x} + \norm{y}$ (triangle inequality)
\end{itemize}
If we further require $ \norm{x} = 0 \implies x = \mathbf{0}$, then the semi-norm becomes a norm.

\end{definition}
Note that any (semi)-norm $\norm{x}$ induces a (semi)-metric $d$ with $d(x,y) = \norm{x-y}$.
}

We denote by $(w_i)_I$ the set of prototypes. Each prototype is assigned to one class. Then $z \in \R^d$ is classified as
\[ f(z)=\argmin_{y=1,\ldots,K} \minop_{i \in I_y} d(z,w_i),\]
where $I_y$ are the prototypes of class $y$. A nearest neighbor classifier (1NN) can also be understood as NPC where one uses the training set as prototypes and thus are not learned. However, by training prototypes one can achieve better classification performance\vv{, and also robustness, see Table~\ref{tab:MNIST-l2},} with less prototypes meaning that NPC are significantly more efficient than 1NN. 
We note that the classification for a point $z$ with label $y$ is correct if
\[ \minop_{i \in I_y} d(z,w_i) - \minop_{j \in I^c_y} d(z,w_j) \vv{<}  0,\]
where $I^c_y$ is the set of all prototypes not belonging to class $y$ (the complement of $I_y$ in $I$).

\subsection{Provable robustness guarantees for semi-metrics}
Next we define the minimal adversarial perturbation of a point $z$ 
for a semi-metric on $\X$, that is the radius $r$ of the smallest ball $B_d(z,r) = \{x \in \X  \,|\, d(x,z) \leq r\} $ around $z$ such at least one point in $B_d(z,r)$ is classified differently than is $z$. If a point $z$ is misclassified then we define the minimal adversarial perturbation to be zero. \vv{We assume that there is a non-empty set of prototypes for every class; thus, there always exists an adversarial example}.
\begin{definition}
The \textbf{minimal adversarial perturbation} $\epsilon_d(z)$ of $z \in \X$ of a NPC using semi-metric $d$ is defined as
\begin{align*}
     \epsilon_d(z) \hspace{-1mm}=\hspace{-1mm} \min\{ r |\hspace{-1mm} \maxop_{x \in B_d(z,r)}\hspace{-1mm} \big(\minop_{i \in I_y}  d(x,w_i) \hspace{-0.25mm}- \hspace{-0.25mm} \minop_{j \in I^c_y} d(x,w_j)\big) \hspace{-0.5mm}\geq\hspace{-0.5mm} 0\}
    \end{align*}
\vv{If $\minop_{i \in I_y} d(z,w_i) - \minop_{j \in I^c_y} d(z,w_j) \geq 0$ then we set $\epsilon_d(z)=0$.}
\end{definition}
In \cite{Saralajew2020fast} they derive for semi-norms
a lower bound on $\epsilon_d$ and in this way get robustness certificates. We generalize this lower bound to semi-metrics which is considerably more general as $\X$ need not be a vector space. It turns out that the only necessary technical requirement  for the proof is the triangle inequality. \vv{This is unlike~\cite{Saralajew2020fast}, where the proof also required the absolute homogeneity of semi-norms.}
\begin{theorem}\label{thm:trivlb}
Let $(\X,d)$ be a semi-metric space, then it holds for the minimal adversarial perturbation $\epsilon_d(z)$ of $z \in \X$ with correct label $y$: 
\[ \epsilon_d(z) \geq \max\left\{0, \frac{\minop_{j \in I_y^c} d(z,w_j)-\minop_{i \in I_y}d(z,w_i)}{2}\right\}.\]
\end{theorem}
\vv{We note that if the semi-metric $d$ can be written as  $d(x,y)=\norm{x-y}$ for some semi-norm $\norm{\cdot}$, then our bound is equal to the one given in~\cite{Saralajew2020fast} }

\begin{figure*}
\begin{center}
\label{fig:TwoMoons-L2}
\begin{tabular}{l@{\.}l@{\.}l}
\includegraphics[width=0.33\textwidth]{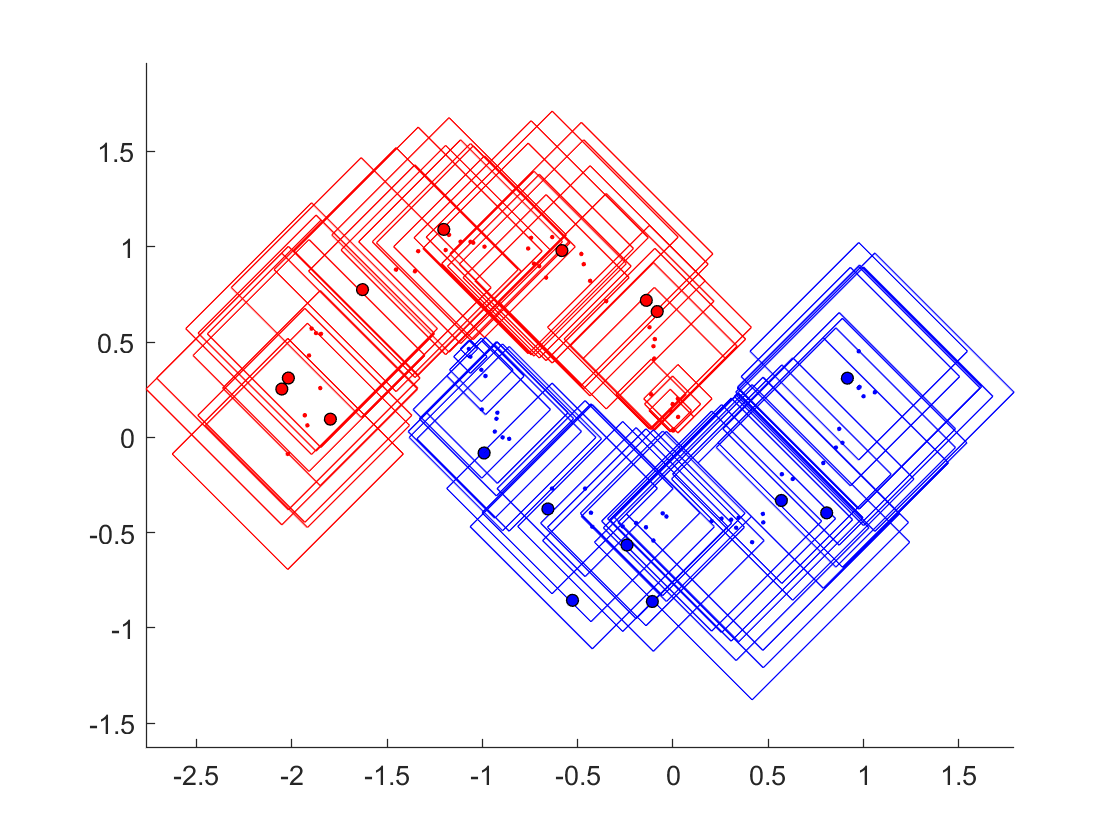}
& \includegraphics[width=0.33\textwidth]{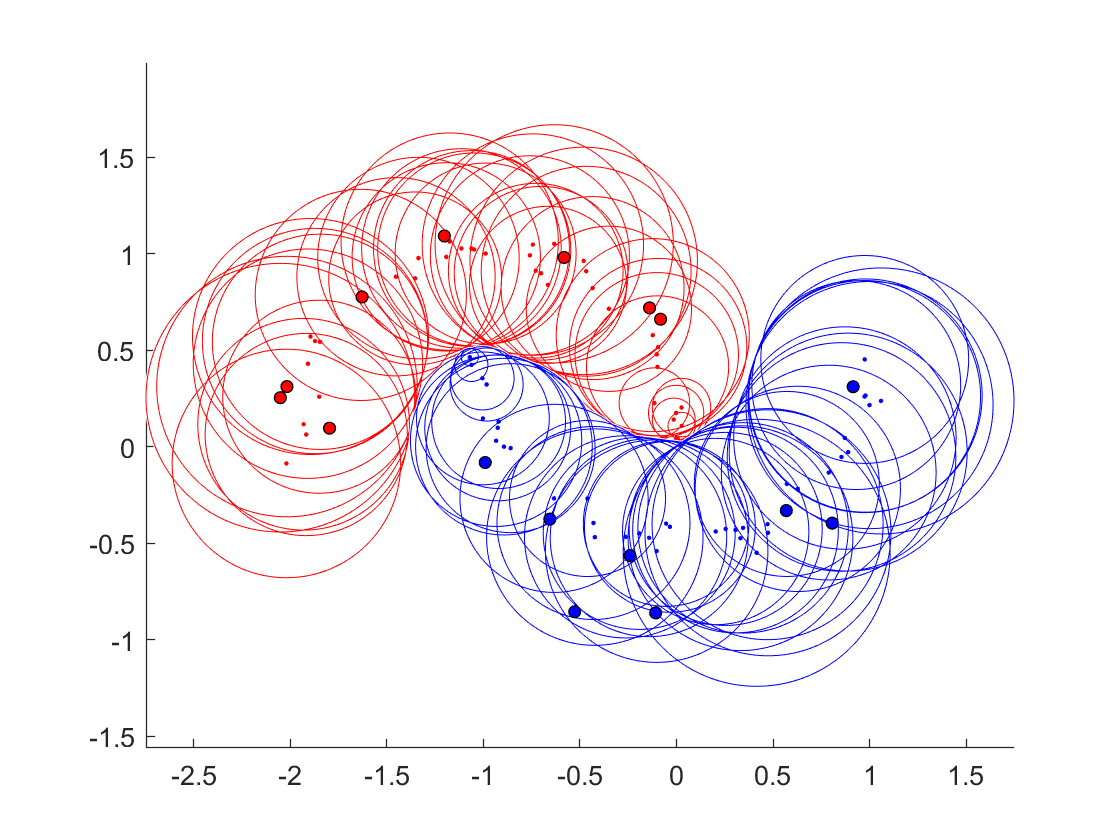}
& \includegraphics[width=0.33\textwidth]{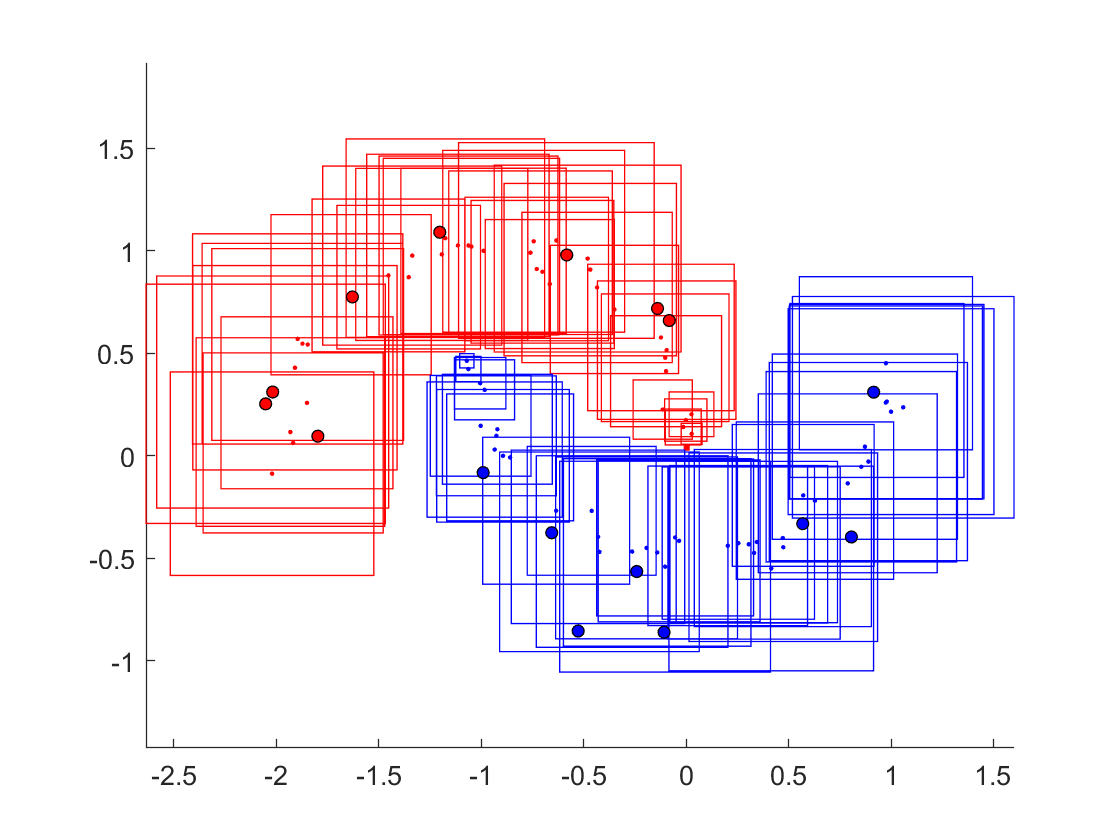}
\end{tabular}
\end{center}
\caption{\textbf{Illustration of the $\ell_q$-minimal adversarial perturbations of a $\ell_2$-NPC for a binary classification problem.} The learned prototypes are shown as the larger red resp. blue dots. For each data point we draw the largest $\ell_1$-(left), $\ell_2$-(middle) and $\ell_\infty$-(right) ball   which is fully classified as the same class. The radii are computed using  Alg. \ref{alg:sketch}. Though there is no specific optimization for multiple-norm robustness, $\ell_2$-NPC possess non-trivial multiple-norm robustness. }
\end{figure*}

\subsection{The minimal adversarial $\ell_q$-perturbation of the $\ell_p$-NPC and lower bounds}
In this section we derive the minimal adversarial $\ell_q$-perturbation for the $\ell_p$-\ours in $\R^d$ where our main interest is $p,q \in \{1,2,\infty\}$. In contrast to the semi-metric case, here we treat the case where the $\ell_q$-metric measuring the size of the adversarial perturbation is different from the $\ell_p$-metric used in the NPC. In this section we use the notation
\[ B_q(x,r)=\{z \in \R^d \,|\, \norm{z-x}_q \leq r\}.\]
Thus we first define
\begin{definition}
The \textbf{minimal adversarial perturbation} $\epsilon_p^q(z)$ of $x \in \X \subset \R^d$ with respect to the $\ell_q$-metric for the $\ell_p$-NPC is defined as:
\begin{align}
   \epsilon^q_p(z)_j=\minop_{r \in \R, x\in\X} & \quad r \nonumber \\
	 \textrm{sbj. to:} & \quad \norm{x-w_i}_p - \norm{x-w_j}_p \geq 0 \nonumber\\
	                   & \quad x \in B_q(z,r) \nonumber
\end{align}

\vv{If $\minop_{i \in I_y} \hspace{-0.25mm}\norm{x-w_i}_p \hspace{-0.5mm}- \hspace{-0.25mm}\minop_{j \in I^c_y} \hspace{-0.25mm}\norm{x-w_j}_p \hspace{-0.25mm}>\hspace{-0.25mm}0$ we set $\epsilon_p^q(z)\hspace{-0.5mm}=\hspace{-0.5mm}0$.}
\end{definition}
The following reformulation of the optimization problem for the computation of the minimal adversarial perturbation $\epsilon^q_p(z)$ allows us to provide a generic and direct way to derive efficiently computable lower bounds on $\epsilon^q_p(z)$.
Note that in the following we always integrate the constraint $x \in \X$ as we will see that this significantly improves the guarantees, e.g. when $\X=[0,1]^d$ in image classification, compared to $\X=\R^d$ as done in \cite{Saralajew2020fast,wang2019evaluating}. 
\begin{theorem}[Exact computation of $\epsilon^q_p(z)$]\label{th:l2-certificate}
Let $z \in \X \subset \R^d$ and denote by $I_y$ the index set of prototypes $(w_j)$ of class $y$ and by $I_y^c$ its
complement (the index set of prototypes not belonging to class $y$). Then define for every $j \in I_y^c$:
\begin{align}\label{eq:rqp}
   r^q_p(z)_j=\minop_{x \in \R^d} & \quad\norm{x-z}_q \\
	 \textrm{sbj. to:} & \quad \norm{x-w_i}_p - \norm{x-w_j}_p \geq 0 \quad \forall \, i \in I_y \nonumber\\
	                   & x \in \X \nonumber
\end{align}
Then $\epsilon^q_p(z)=\minop_{j \in I_y^c} r^q_p(z)_j.$
\end{theorem}
While the corresponding optimization problems are \vv{often} non-convex, we will see in the following that they are equivalent to convex optimization problems in the case where the $\ell_2$-distance is used in the NPC ($p=2$). Using the formulation of the exact problem as an optimization problem we can now simply derive lower bounds on $\epsilon_p^q(z)$ by relaxing the optimization problem \eqref{eq:rqp}.

We consider for this reason the following optimization problems. For $i \in I_y$ and $j \in I^c_y$ we define:
\begin{align}\label{eq:rhopq}
 \rho^q_p(z)_{i,j}=\minop_{x \in \R^d} & \quad\norm{x-z}_q \\
	 \textrm{sbj. to:} & \quad \norm{x-w_i}_p - \norm{x-w_j}_p \geq 0 \nonumber\\
	                   & x \in \X \nonumber
\end{align}
In Theorem~\ref{thm:rhocomplex} we show that these simpler problems can often be solved efficiently, although the computation of $\epsilon_p^q$ is often intractable, as we show in Theorem~\ref{thm:rcomplex}.

\begin{table}
\begin{center}
\begin{tabular}{c|c| c | c| c|} 
 & & \multicolumn{3}{c|}{$\ell_q$-threat model}\\
 \cline{2-5}
\multirow{4}{*}{ \begin{turn}{90}\hspace{+1mm} $\ell_p$-distance \end{turn}}& & $\ell_1$ & $\ell_2$ & $\ell_\infty$\\ 
\cline{2-5}
 & $\ell_1$ & NP-hard & NP-hard & $O(d \log(d))$ \\ 
 \cline{2-5}
  & $\ell_2$ & $\Theta(d)$ & $\Theta(d)$ & $\Theta(d)$ \\
 \cline{2-5}
 & $\ell_\infty$ & $\Theta(d)$ & $O(d \log(d))$ & $\Theta(d)$ \\
\cline{2-5}
\end{tabular}
\end{center}
\caption{\label{tab:hardness_rho} Computational complexity of $\rho_p^q(z)_{i,j}$.} 
\end{table}
\begin{theorem}\label{thm:rhocomplex}
The computational complexities of optimization problems $\rho_p^q(z)_{i,j}$ for $p,q \in \{1,2,\infty\}$ for $\X=\R^d$ are summarized in Table~\ref{tab:hardness_rho}.
\end{theorem}

\begin{theorem}\label{thm:rcomplex}
The computational complexities of optimization problems $r_p^q(z)_j$ in \eqref{eq:rqp} for $p,q \in \{1,2,\infty\}$ and $\X=[0,1]^d$ are summarized in Table~\ref{tab:hardness_eps}.

\begin{table}
\begin{center}
\begin{tabular}{c|c| c | c| c|} 
  & & \multicolumn{3}{c|}{$\ell_q$-threat model }\\
 \cline{2-5}
\multirow{4}{*}{ \begin{turn}{90}\hspace{+1mm} $\ell_p$-distance \end{turn}}  & & $\ell_1$ & $\ell_2$ & $\ell_\infty$\\ 
 \cline{2-5}
& $\ell_1$ & NP-hard & NP-hard & Poly \\ 
 \cline{2-5}
 &$\ell_2$ & Poly & Poly & Poly \\
 \cline{2-5}
 &$\ell_\infty$ & NP-hard & NP-hard &NP-hard \\
\cline{2-5}
\end{tabular}
\end{center}
\caption{ Computational Complexity of $r_p^q(z)$ and $\epsilon_p^q(z)$.} 
\label{tab:hardness_eps}
\end{table}
\end{theorem}
 Apart from the known $\ell_2$-case (see \cite{wang2019evaluating}) we show that also $\ell_1$-NPC can be certified efficiently for the $\ell_\infty$-threat model. Because of this theorem it is even more important that at least for the $\ell_\infty$-NPCs efficient lower bounds are available for all threat models in $q=\{1,2,\infty\}$. We note that the optimization problem for $r^q_2(z)_j$ in \eqref{eq:rqp} is equivalent to a quadratic program for $q=2$ and to a linear programs for $q \in \{1,\infty\}$  for both with and without box constraints. 

The following lemma shows that \eqref{eq:rhopq} can be used to get a lower bound on the minimal adversarial perturbation, and subsequently we show that it improves on the previous bound given in Theorem \ref{thm:trivlb} which has been derived  by \cite{Saralajew2020fast}. In particular, this bound can be tight and we show in Table \ref{tab:time} in Section \ref{sec:exp} that this happens frequently in practice and thus allows to avoid the significantly more complex problems in \eqref{eq:rqp}.
    \begin{lemma}\label{lem:minmaxrho}
    It holds 
    \[ \epsilon^q_p(z) \geq \minop_{j \in I^y_c}\maxop_{i \in I_y} \rho^q_p(z)_{i,j}.\]
    Moreover, let $(j^*,i^*)$ be the prototype pair in $I^y_c \times I_y$ which realizes the lower bound and denote by $x^*$ the minimizer of $\rho^q_p(z)_{i^*,j^*}$. Then if $x^*$ fulfills
    \[ \norm{x^*-w_i}_p - \norm{x^*-w_{j^*}}_p \geq 0 \quad \forall i \in I_y,\]
    then $\epsilon^p_q(z)=\minop_{j \in I^y_c}\maxop_{i \in I_y} \rho^q_p(z)_{i,j}.$
    \end{lemma}


\begin{theorem}\label{th:betterbound}
   The lower bound on $\epsilon_p^p(z)$ of Lemma~\ref{lem:minmaxrho} is at least as good as the one of Theorem \ref{thm:trivlb}. That is,
   \begin{align*}
       \minop_{j \in I^y_c}\maxop_{i \in I_y} & \, \rho^p_p(z)_{i,j} \geq \minop_{j \in I^y_c}\rho^p_p(z)_{i^*,j} \\ \geq  &
\max\left\{0, \frac{\minop_{j \in I_y^c} \norm{z-w_j}_p-\minop_{i \in I_y}\norm{z-w_i}_p}{2}\right\},
   \end{align*}
   where $i^* \in \arg\minop_{i \in I_y} \norm{z-w_i}_p$. 
\end{theorem}

In order to be able to use these lower bounds for certified training of our \ours, their efficient computation is of high importance which we discuss next. 




For better intuition we discuss some cases in more detail. The $\ell_2$-NPC have a nice geometric descriptions as the set 
\begin{align*}
  &  \{z |\norm{z-w_i}_2=\norm{z-w_j}_2\}\\
= &\{z | \inner{w_j-w_i,z} + \frac{\norm{w_i}^2_2-\norm{w_j}^2_2}{2}=0\}
\end{align*}is a hyperplane. 
Thus the computation of $\rho^q_2(z)_{i,j}$ for $\X=\R^d$ corresponds to the computation of the $\ell_q$-distance of a point to a hyperplane:
\[ \rho^q_2(z)_{i,j}=\frac{\norm{z-w_j}_2^2 - \norm{z-w_i}^2_2}{2\norm{w_i-w_j}_{q^*}},\]
where $q^*$ denotes the dual norm of $q$. This has also been derived in \cite{wang2019evaluating}. 
As illustration  how the constraints $\X=[0,1]^d$, e.g. in image classification, improve the certificates, we show in Figure \ref{fig:box} the ball which can be certified in $\R^d$ resp. $[0,1]^d$. 

\begin{figure}[t]
\begin{center}
\includegraphics[width=0.95\columnwidth]{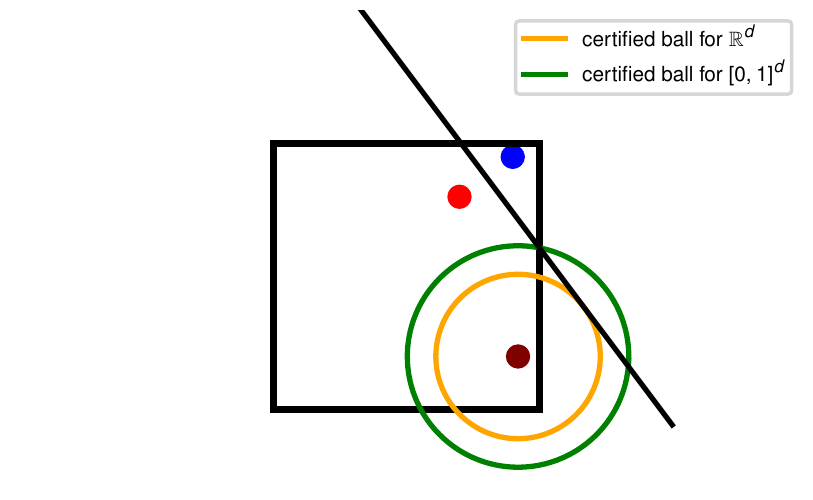}
\caption{\label{fig:box}Illustration for $\ell_2$-NPC for two prototypes (red and blue): when taking into account that the data lies in $[0,1]^d$ we can certify a larger ball than in $\R^d$.}
\end{center}
\end{figure}

\subsection{How to do the certification efficiently}
Table \ref{tab:hardness_rho} shows that $\rho^q_p(z)_{i,j}$ can be computed efficiently or even given in closed form except for the two cases when $(p,q) \in \{(1,1), (1,2)\}$. However, that would still mean that the lower bound of Lemma \ref{lem:minmaxrho}
\[ \epsilon^q_p(z) \geq \minop_{i \in I_y}\maxop_{j \in I^c_y} \rho^q_p(z)_{i,j},\]
would require us to solve naively $|I_y||I_y^c|$ such problems. Seemingly, the bound in Theorem \ref{thm:trivlb} is much cheaper as it requires only $(|I_y|+|I_y^c|)$  operations even though one has to note that the bound only exists for the case when $p=q$.

\textbf{i) A lower bound:} Theorem \ref{th:betterbound} shows that
when fixing $i^*=\argmin_{i \in I_y} \norm{z-w_i}$ and then computing
\[ \min_{j \in I^y_c} \rho^{q}_p(z)_{i^*,j},\]
 yields by Lemma \ref{lem:minmaxrho} a lower bound on $\epsilon^q_p(z)$. By Theorem \ref{th:betterbound} this lower bound is for the case $p=q$ still better than the one of Theorem \ref{thm:trivlb} while having the same complexity of $|I_y|+|I_y^c|$ operations.
Obviously, when integrating box constraints, that is $\X=[0,1]^d$, the gap can only become larger between the two bounds. 

\textbf{ii) Using simpler lower bounds:} When certifying bounds for $\X=[0,1]^d$ we first compute the lower bounds for $\X=\R^d$ as they are often available in closed form and are definitely lower bounds for the more restricted case $\X=[0,1]^d$. By fixing again $i^*$ we can then use $s_j:=\rho^q_p(z)_{i^*,j}$ and define the minimum and minimizer as $(\lambda,j^*)=\min_{j \in I_y^c} \rho^{q}_p(z)_{i^*,j}$. Now, let us denote by $\kappa^q_p(z)_{i^*,j}$ the corresponding quantity when using $\X=[0,1]^d$ instead of $\X=\R^d$. Then we only need to compute $\kappa^q_p(z)_{i^*,j}$ if $s_j<\kappa^q_p(z)_{i^*,j^*}$, which is typically satisfied for very few instances, so most computations are pruned.

\textbf{iii) Dual problems:} as in \cite{wang2019evaluating} we use the dual problems when computing $r^q_2(z)_j$. This has three advantages. First, we always get a lower bound using weak duality, second, we stop solving $r^q_p(z)_j$ when the dual value is higher than our currently smallest upper bound and third; empirically only few constraints of the problems become active; thus, the solutions are dual-sparse.

\textbf{Final Certification:} in Algorithm \ref{alg:sketch} we sketch the certification process. It does not include all details (see above) which we use for speeding up the computation of lower bounds as well as the exact minimal adversarial perturbation.

\begin{algorithm}[tb]
   \caption{Sketch of certification algorithm for correctly classified point $z$}
   \label{alg:sketch}
\begin{algorithmic}
   \STATE \textbf{// Computation of $\lambda$ as lower bound on $\epsilon^q_p(z)$}\\
  \STATE $i^*=\argmin_{i \in I_y}\norm{z-w_i}_p$
   \STATE $s_j = \rho^q_p(z)_{i^*,j}$, $j \in I^c_y$ ($s_j$ lower bounds $r^q_p(z)_j$)
  \STATE $(\lambda,j^*)=\min_{j \in I_y^c}\rho^q_p(z)_{i^*,j}$
  \IF{minimizer $x^*$ of $\rho^q_p(z)_{i^*,j^*}$ is feasible for $r^q_p(z)_{j^*}$}
     \STATE $\epsilon^q_p(z)=\lambda$
     and return
  \ELSE
     \STATE $\lambda$ is lower bound on $\epsilon^p_q(z)$
  \ENDIF
  \STATE{// \textbf{Computation of $\epsilon^q_p(z)$ ( $p=2$ or $(p,q)=(1,\infty)$)}}
  \STATE $\mu=r^q_p(z)_{j^{*}}$ // (it holds $\mu\geq \epsilon^q_p(z)$)
  \FOR{$j=1$  {\bfseries to} $|I_y^c|$}
  \IF{$s_j < \mu$}
     \STATE compute $r^q_p(z)_j$
     \IF{$r^q_p(z)_j< \mu$}
      \STATE $\mu=r^q_p(z)_j$
     \ENDIF
  \ENDIF
  \STATE $\epsilon^q_p(z)=\mu$
  \ENDFOR
\end{algorithmic}
\end{algorithm}

\section{Perceptual Metric}\label{sec:perceptual}
The hypothesis underlying the goal of adversarial robustness is that images which have the same semantic content, should be classified the same
(with the exception at the true decision boundary). However, this would require a human oracle which judges if the semantic
content is similar. A proxy is the typical $\ell_p$-threat model, where for suitable chosen radius $\epsilon_p$ one expects that for a given
image $x$ also 
$B_p(x,\epsilon_p)$
 should be classified
the same as for humans the resulting images are (semantically) indistinguishable from the original image. However, it is well known that pixel-based
$\ell_p$-distances are not a good measure of image similarity. A huge literature in computer vision  discusses the construction of metrics which better correspond to human perception of similarity of images e.g. the SSIM metric of \cite{wang2004image}. More recently, neural perceptual metrics, such as the LPIPS distance,  have been proposed in \cite{zhang2018unreasonable}. The LPIPS distance is based on a feature mapping of a fixed neural network and has been shown to correlate better with human perception \cite{zhang2018unreasonable,laidlaw2021perceptual}. In \cite{laidlaw2021perceptual} it has been used as threat model in adversarial training. Moreover, \cite{kireev2022effectiveness} have shown that the LPIPS distance better correlates with the severity level of common corruptions than the $\ell_2$-distance.
Moreover, $\ell_p$-distance based NPC are not competitive for CIFAR10. These two aspects motivate us to investigate the perceptual metric-based \ours as well as novel techniques for the certification in the LPIPS-threat model.

\paragraph{The perceptual metric:} Given the output $g^{(l)}(x) \in \R^{H_l \times W_l \times C_l}$ of the $l$-th layer of a fixed neural network (we use Alexnet as suggested by \cite{zhang2018unreasonable}) of height $H_l$ and width $W_l$ and channels $C_l$, we define the normalized output of a layer as $\hat{g}_{h,w}^{(l)}(x)=\frac{g_{h,w}^{(l)}(x)}{\norm{g_{h,w}^{(l)}(x)}_2}$. The LPIPS distance $d$ is then defined in \cite{zhang2018unreasonable} as
\[ d^2(x,y) = \sum_{l \in I_L} \frac{1}{H_l W_l} \sum_{h,w} \norm{w_l \odot  \left(\hat{g}^{(l)}_{h,w}(x)-\hat{g}^{(l)}_{h,w}(y)\right)}_2^2,\]
where the weights $w_l$ are learned using human perception data and $I_L$ is the index set of layers used for the metric. We follow \cite{laidlaw2021perceptual} and use the unweighted \vv{(i.e., weights perform an identity mapping)} version in order to be able to directly compare to them. However, it would be easy to adapt our approach for the weighted version. We define the embedding, $\phi:[0,1]^d \rightarrow \R^D$
\begin{equation} x \mapsto \phi(x)=\left(\frac{\hat{g}^{(l)}}{\sqrt{H_l W_l}}\right)_{l \in I_L},
\end{equation}
so that the unweighted LPIPS distance can simply be written as a standard Euclidean distance $d(x,y)=\norm{\phi(x)-\phi(y)}_2$ in the embedding space.

The mapped image space $\phi(I)$ of all natural images $I$ is a subset of $\phi([0,1]^d)$, which can be seen as an at most $d$-dimensional continuous ``submanifold'' of the embedding space $\R^D$. Thus for all points $z  \in \R^D \backslash \phi([0,1]^d)$ there exists no pre-image in $[0,1]^d$. However, the Euclidean distance between every mapped images $x,y \in I$ corresponds to the perceptual distance between them. Thus we train our \ours in the embedding space $\R^D$ and certify it with respect to the Euclidean distance which in turn yields guarantees with respect to the LPIPS distance.

\begin{figure}[ht]
\begin{center}
\includegraphics[width=0.95\columnwidth]{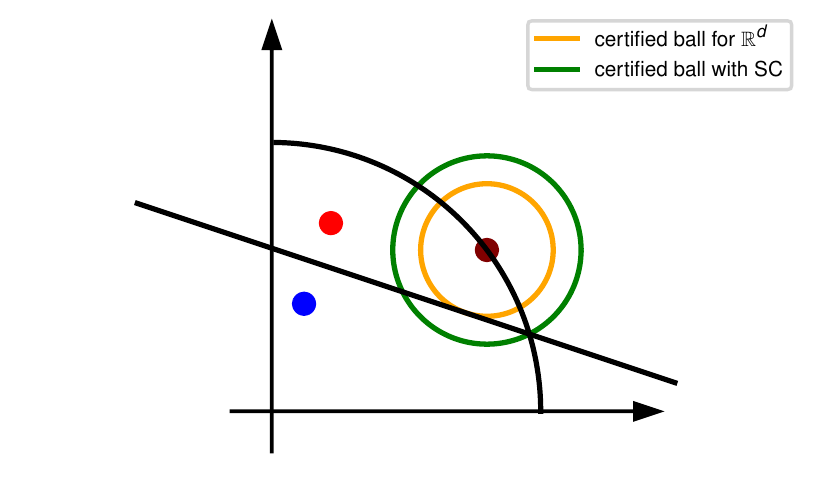}
\caption{\label{fig:SC}The embedded data $\phi(x)$ lies on the intersection of the positive orthant and the sphere (shown in black). In the embedding space the $\ell_2$-metric corresponds to the perceptual metric. Taking these non-negative spherical constraints (SC) into account we can certify a much larger ball than using only the standard certification in $\R^d$.}
\end{center}
\end{figure}

\subsection{Certification in the Perceptual threat model}
Up to our knowledge this is the first paper showing results for certification with respect to this threat model aligned with human vision. We can use all techniques we have discussed in Section \ref{sec:proto} as we are working with a Euclidean distance in $\R^D$. However, we have more knowledge about $\phi([0,1]^d)$ as the output of each layer is normalized so that $\phi(x)$ lies on a product of spheres with radius $r_l=\frac{1}{\sqrt{H_l W_l}}$ as
\begin{equation}\label{eq:lpips}
\norm{\phi^{(l)}_{h,w}(x)}_2 =\norm{\frac{\hat{g}_{h,w}^{(l)}}{\sqrt{H_l W_l}}}_2=\frac{1}{\sqrt{H_l W_l}}:=r_l, \quad 
\end{equation}
for any $l\in I_L, h\in I_H, w \in I_W$. Additionally, we know due to the structure of Alexnet that $\phi_l(x)$ is non-negative for all layers, see Figure \ref{fig:SC} for an illustration. While we can integrate some of the properties of the mapping $\phi$ into the certification, it is computationally intractable to use as constraint $x \in \phi([0,1]^d)$. Thus our certification works on an overapproximation of $\phi([0,1]^d)$ and thus yields lower bounds on the minimal adversarial perceptual distance.

Basically, we can write our constraints 
in $\R^D$ as
\vv{\begin{align}
    \X =& \X_1 \times \cdots \times \X_L \\
    &\X_l = \left(\frac{1}{\sqrt{H_lW_l}}S^{c_l} \cap [0,\infty)^{c_l} \right)^{H_lW_l},\; l=1,\ldots,L, \nonumber
\end{align}}
where $c_l$ is the number of channels in layer $l$ of the output of the layer $l$ and $D=\sum_{l=1}^L H_lW_lc_l$. \vv{We use upper indices (e.g., $x^{(h,w,l)}$) to denote slice of vector $x$ which corresponds to vector of channels at position $h,w$ in layer $l$. The constants $r_l$ for $1 \leq l \leq L$ were defined in~\eqref{eq:lpips}.}

As we use $\ell_2$-NPC we have to compute:
\begin{align}\label{eq:r-lpips}
   \rho(z)_{i,j}=\minop_{x \in \R^D} & \quad\norm{x-z}_2\\
	 \textrm{sbj. to:} & \quad \inner{x,w_j-w_i} + \frac{\norm{w_i}^2_2-\norm{w_j}^2_2}{2} \geq 0 \nonumber\\
	                   & \quad \norm{x^{(h,w,l)}}^2_2=r_l^2, \;  l=1,\ldots,L\nonumber\\
	                   &~~~~~~~~~~~~~~~~~~~~~~~~~~~~~~~~~~~  h = 1,\ldots,H_l\nonumber\\
	                   &~~~~~~~~~~~~~~~~~~~~~~~~~~~~~~~~~~~  w = 1,\ldots, W_l\nonumber\\
	                   & \quad x_d \geq 0,\; ~~~~~~~~~~~~~~~~~d=1,\ldots,D. \nonumber
\end{align}
Despite this problem is non-convex due to the quadratic equality constraints we can derive a convex dual problem (we derive it for an equivalent problem) which is sufficient to provide us with lower bounds using weak duality.
\begin{proposition}\label{pro:lb-lpips}
\vv{Define $v = w_j-w_i$ and $b = \frac{\norm{w_i}^2_2-\norm{w_j}^2_2}{2}$.}
A lower bound on the optimal value of the optimization problem \eqref{eq:r-lpips} is given by
\[ \sqrt{2L  + 2\left(\maxop_{\lambda \geq 0} - \sum_{h,w,l} \norm{\left(z^{(h,w,l)} - \lambda v^{(h,w,l)}\right)^+}_2r_l + \lambda b\right)} \]
which can be efficiently computed using bisection.
The summation $\sum_{h,w,l}$ is a shortcut for $\sum_{1 \leq l \leq L}\sum_{1 \leq h \leq H_l}\sum_{1 \leq w \leq W_l}.$
\end{proposition}
In the experimental results in Figure \ref{fig:rcurve-lpips} one can clearly see that using this lower bound improves significantly over the standard lower bound of Lemma~\ref{lem:minmaxrho}.


\section{Efficient Training of \ours}\label{sec:training}
In this section we describe the training procedure for our \ours. The key advantage compared to the work of \cite{Saralajew2020fast} is that despite our lower bounds, see Theorem \ref{th:betterbound}, are better and often tight, they can be computed with the same time complexity as theirs if $p \in \{2, \infty\}$. Thus we can do efficient certified training. As objective we use the capped sum of the lower bounds:

\[ \maxop_{(w_i)_{i \in I}} \frac{1}{n}\sum_{r=1}^n\min\Big\{\minop_{j \in I_y^c}\maxop_{i \in I_y}\rho^q_p(z_r)_{i,j},R\Big\},\]

where we recall the definition of $\rho^q_p$ from~\ref{eq:rhopq}:

\begin{align}
 \rho^q_p(z)_{i,j}=\minop_{x \in \R^d} & \quad\norm{x-z}_q \\
	 \textrm{sbj. to:} & \quad \norm{x-w_i}_p - \norm{x-w_j}_p \geq 0 \nonumber\\
	                   & x \in \X \nonumber
\end{align}

and $R$ is an upper bound on the margin we want to enforce.
The cap is introduced in order to avoid that single training points have excessive margin at the price of many others having small margin; in turn, it is equivalent to minimizing hinge-loss. \vv{The loss is minimized via stochastic gradient descent resp. ADAM with large batch sizes.} Note further that, for misclassified points we use a signed version of  $\rho^q_p(z_r)_{i,j}$ by flipping the constraint in \eqref{eq:rhopq} and using 
$-\rho^q_p(z_r)_{j,i}$ instead,
which can be interpreted as signed distance to the decision boundary.
Doing this has the advantage that we get gradient information from all points. 
Maximizing our objective has a direct interpretation in terms of maximizing robust accuracy or more precisely the area under the robustness curve capped at radius $R$. This is in contrast to \cite{Saralajew2020fast} who use as loss their lower bound divided by the sum of the distances where this interpretation is due to the rescaling not applicable.

\section{Experiments}\label{sec:exp}
\vv{The code for experiments is available in our repository\footnote{\href{https://github.com/vvoracek/Provably-Adversarially-Robust-Nearest-Prototype-Classifiers}{https://github.com/vvoracek/Provably-Adversarially-Robust-Nearest-Prototype-Classifiers}.} where we also provide the training details.}
We first evaluate the improvements in the certification of better lower bounds resp. exact computation compared to the ones of \cite{Saralajew2020fast} as well as \cite{wang2019evaluating}. In a second set of experiments we compare our $\ell_p$-\ours to the $\ell_p$-NPC of \cite{Saralajew2020fast} resp. to nearest neighbor classification as well as deterministic and probabilistic certification techniques for neural networks on MNIST and CIFAR10 (see App. \ref{app:cifar}). Finally, we discuss our NPC using the perceptual metric and its certification where there is no competitor as up to our knowledge this is the first paper providing robustness certificates.  \vv{The training time is about a few hours on a laptop.}

\textbf{Comparison of our lower bounds:}
One of the major contributions of this paper are our efficient lower bounds on the minimal adversarial perturbation $\epsilon^q_p(z)$. They can be computed so fast that it is feasible to use them during training. We show in Table \ref{tab:lbounds} that our $\ell_q$-bounds improve significantly over the ones of \cite{Saralajew2020fast} (Th. \ref{thm:trivlb}, $\X=\R^d$), which only work if $p=q$ and \cite{wang2019evaluating} (Lemma~\ref{lem:minmaxrho}, $\X=\R^d$, see \eqref{eq:l2-lb} for $p=2$)  as we are the only ones who integrate box constraints (Lemma~\ref{lem:minmaxrho}, $\X=[0,1]^d$). In Table~\ref{tab:lbounds}, we show that for  the $\ell_1$-, $\ell_2$- and $\ell_\infty$ threat models   our lower bounds are very close to the exact values. The computation of these lower bounds takes for the \textbf{full} test set of MNIST: $\ell_1$: 188s, $\ell_2$: 33s, $\ell_\infty$: 131s. This is two orders of magnitude faster than the computation of the exact bounds in Table \ref{tab:time}.
\begin{table}[t]
\caption{\textbf{Lower bounds on $\epsilon^q_p(z)$.} Mean of the lower bounds of \cite{Saralajew2020fast} (Theorem~\ref{thm:trivlb}), the lower bounds of \cite{wang2019evaluating}) in \eqref{eq:l2-lb} ($\X=\R^d$), our lower bounds integrating $\X=[0,1]^d$ and the exact radius on the test set for $\ell_2$-NPC for $\ell_1$-,$\ell_2$- and $\ell_\infty$-threat model.}\label{tab:lbounds}
\vspace{-2mm}
\begin{center}
\begin{small}
\setlength{\tabcolsep}{1.5pt}
\begin{tabular}{l|c|c|ccc|c|}
      &         &     &  \multicolumn{3}{c|}{Lower bounds} & Exact          \\
Model & Num.& Threat & Th. 2.2 & Th. 2.6  & Th. 2.6 & radius\\
     &  Proto.        &  model      & $\R^d$ & $\R^d$ & $[0,1]^d$ & $[0,1]^d$ \\
\midrule
\multirow{2}{*}{$\ell_2$-\ours} & \multirow{3}{*}{4000} & $\ell_1$ & - & 9.71 & 11.77 & 12.11\\
& & $\ell_2$ & 0.39 & 1.86 & 1.96 & 1.99 \\
MNIST&
& $\ell_\infty$ & - & 0.14 & 0.16 & 0.17\\
\bottomrule
\end{tabular}
\end{small}
\end{center}
\end{table}
For our $\ell_\infty$-NPC and $\ell_\infty$-threat model we get mean lower bounds of $0.3545$ for \cite{Saralajew2020fast}, $0.3560$ for the ones from \eqref{eq:linf-linf} with $\X=\R^d$, and $0.3616$ for ours from Lemma~\ref{lem:minmaxrho} with $\X=[0,1]^d$ in~\eqref{eq:linf_linf_box}. Here the differences are  smaller than for the $\ell_2$-NPC.

\textbf{Time for certification:}
The computation of the exact minimal adversarial perturbation is only feasible for relatively small neural networks \cite{TjeTed2017} and for ensemble of decision trees \cite{kantchelian2016evasion}. Both use mixed-integer formulations which do not scale well. For boosted decision stumps one can compute the exact robust accuracy \cite{andriushchenko2019provably}. However, the computation of the exact robust accuracy is already considerably easier than the minimal adversarial perturbation. For $\ell_2$-NPC we can compute the exact adversarial perturbation for the $\ell_1$-, $\ell_2$-, and $\ell_\infty$-threat model. In Table \ref{tab:time} we report the certification time per point and other statistics for our $\ell_2$-\ours prototypes on MNIST with 400 prototypes per class (ppc) and the $\ell_2$-GLVQ -model of \cite{Saralajew2020fast} on CIFAR10 with 128 ppc. \vv{We can also produce weaker certificates faster. For instance, using Lemma~\ref{lem:minmaxrho}, we can certify MNIST robust accuracy $67\%$ in under $2s$ instead of the exact $73\%$ reported in Table~\ref{tab:MNIST-l2}.}

Regarding the model of $\ell_2$-GLVQ on CIFAR10, we have an accuracy of $48.6\%$ (which corresponds to 4859 correctly classified test points). Of these ones
we can solve between 72.2\% for $\ell_\infty$ and $75.8\%$ for $\ell_1$ directly using Lemma \ref{lem:minmaxrho} by checking the condition after the computation of the lower bounds. This shows the usefulness of Lemma \ref{lem:minmaxrho} as it avoids a lot of QPs ($\ell_2$) rsp. LPs ($\ell_1,\ell_\infty$) to be solved. Next we see that the number of LPs/QPs needed to be solved per point is less than 1.43 which has to be compared to the worst case of $|I_y^c|=1152$. This shows that our prior reduction using our tight lower bounds integrating box constraints helps  to significantly reduce the number of problems $r^q_p(z)_j$ which need to be solved. In total we get certification times between $0.25s$ ($\ell_2$) and $0.9s$ ($\ell_\infty$) per point which allows us to do the exact certification for all three threat models.

\begin{table}[t]
\caption{\textbf{Time/Statistics for exact minimal adversarial perturbation} for $\ell_2$-NPC}\label{tab:time}
\vspace{-2mm}
\begin{center}
\begin{small}
\setlength{\tabcolsep}{1.5pt}
\begin{tabular}{l|c|c|cccc|}
Model & Num.& Threat & Direct & Total  & QP/LP& Cert. Time \\
     &  Proto.        & model & solved & QP/LP & per pt & per pt\\
\midrule
\multirow{2}{*}{$\ell_2$-\ours} & \multirow{3}{*}{4000} & $\ell_1$ & 4261 (43.8\%) &  10195 & 1.86 & 0.54s\\
& 
& $\ell_2$ & 3170   (32.6\%) & 11630 & 1.77& 0.49s\\
MNIST&
& $\ell_\infty$ & 2073 (21.3\%) & 21081 & 2.75 & 1.3s\\
\midrule
\multirow{2}{*}{$\ell_2$-GLVQ} & \multirow{3}{*}{1280} & $\ell_1$ &  3683 (75.8\%) & 1817 & 1.54 &  0.76s\\
& 
& $\ell_2$ & 3546 (73.0\%) & 1777 & 1.35 & 0.25s\\
CIFAR10&
& $\ell_\infty$ & 3511 (72.2\%) & 1933 & 1.43 &  0.9s\\
\bottomrule
\end{tabular}
\end{small}
\end{center}
\end{table}

\textbf{Evaluation of our NPC:}
We report certified robust accuracy (CRA) and upper bounds on robust accuracy (URA), e.g. computed via an adversarial attack, on MNIST and CIFAR10 (in App. \ref{app:cifar}) for \ours and the GLVQ of \cite{Saralajew2020fast}. For $\ell_2$-NPC CRA and URA are equal as we compute exact adversarial perturbations. As an interesting baseline, we report results for the one nearest neighbor classifier (1NN). Additionally, we compare to deterministic and probabilistic certification techniques of neural networks.

\begin{table}[t]
\caption{\textbf{MNIST:} lower (CRA) and upper bounds (URA) on $\ell_2$-robust accuracy for $\ell_2$-NPC}\label{tab:MNIST-l2}
\vspace{-2mm}
\begin{center}
\begin{small}
\setlength{\tabcolsep}{2pt}
\begin{tabular}{l|c|cc|cc|cc|}
MNIST & std. & \multicolumn{2}{c|}{$\epsilon_2=1.5$} & 
\multicolumn{2}{c|}{$\epsilon_2=1.58$} & 
\multicolumn{2}{c|}{$\epsilon_2=2$} \\
& acc. & CRA & URA & CRA & URA & CRA & URA\\
\midrule
$\ell_2$-\ours & 97.3 & \textbf{75.5} & 75.5 & \textbf{73.0} & 73.0 & \textbf{56.1} & 56.1\\
$\ell_2$-GLVQ 
& 95.8 & 69.7 & 69.7 & 67.1 & 67.1 & 53.5 & 53.5\\
1-NN 
& 96.9 & 52.1 & 52.1 & 47.3 & 47.3 & 23.7 & 23.7\\
\midrule
GloRob 
& 97.0 & - & - & 62.8 & 81.9 & - & - \\
OrthConv 
& \textbf{98.1} & - & - & 61.0 & 75.5 & - & - \\
LocLip 
& 96.3 & - & - & 55.8 & 78.2 & - & - \\
BCP 
& 92.4 & - & - & 47.9 & 64.7 & - & -\\
CAP 
& 88.1 & - & - & 44.5 & 67.9 & - & - \\
\midrule
SmoothLip$_{\sigma=0.5}$ 
&  \textbf{98.7} & \textbf{81.8$^*$} & - & - & -  & 0$^*$ & -\\
SmoothLip$_{\sigma=1}$ 
&  93.7 & 62.7$^*$ &  - & - & -  & 44.9$^*$ & -\\
\bottomrule
\end{tabular}
\end{small}
\end{center}
\vskip -0.1in
\end{table}

\textbf{MNIST - $\ell_2$-NPC:} In Table \ref{tab:MNIST-l2} we show the results for the \textbf{$\ell_2$-threat model} on MNIST. Our $\ell_2$-\ours outperforms the $\ell_2$-GLVQ 
for all $\epsilon_2$. The values for $\epsilon_2$ were chosen 
according to the neural network literature. Note that our $\ell_2$-\ours outperforms all deterministic methods: GlobRob \cite{leino2021globallyrobust}, OrthConv~\cite{singla2022improved}, LocLip \cite{huang2021training}, BCP \cite{Lee2020Lipschitz}
and CAP \cite{WonEtAl18} in terms of certified robust accuracy and often in the terms of clean accuracy. \vv{For the details on comparison with~orthogonal convolutions, see Appendix~\ref{app:SOC}}. The randomized smoothing approach SmoothLip of \cite{jeong2021smoothmix} outperforms us for $\sigma=0.5$ in terms of clean accuracy and robust accuracy at $\epsilon_2=1.5$ but their robust accuracy at $\epsilon_2=2$ is zero, whereas we have $56.1\%$ exact robust accuracy. Their second model with $\sigma=1$ which is able to certify also larger radii is in all aspects worse than our $\ell_2$-\ours. This shows that our certified prototype classifiers can challenge neural networks in terms of certified robust accuracy. Moreover, 
\cite{Saralajew2020fast} report for their $\ell_2$-GLVQ  a certified robust accuracy of $34.4\%$ at $\epsilon=1.58$ whereas with our exact computation we get that their exact robust accuracy is 67.1\%. This shows the quality of our exact certification techniques. With our certified training \ours has $6\%$ better robust accuracy  and $1.5\%$ better standard accuracy ($97.3\%$ vs. $95.8\%$) than $\ell_2$-GLVQ. 

\begin{table}[t]
\caption{\textbf{MNIST:} lower (CRA) and upper bounds (URA) on robust accuracy for multiple threat models for our $\ell_2$-\ours, the $\ell_2$-NPC of \cite{Saralajew2020fast}, a 1-NN classifier. As comparison we show MMR-Univ of \cite{croce2020provable} which is a neural network specifically trained for certifiable multiple-norm robustness.\label{tab:MNIST-mult}}
\begin{center}
\begin{small}
\setlength{\tabcolsep}{2pt}
\begin{tabular}{l|c|cc|cc|cc|cc|}
MNIST & std. & \multicolumn{2}{c|}{$\epsilon_1=1$} & 
\multicolumn{2}{c|}{$\epsilon_2=0.3$} & 
\multicolumn{2}{c|}{$\epsilon_\infty=0.1$} &\multicolumn{2}{c|}{union}\\
& acc. & CRA & URA & CRA & URA & CRA & URA & CRA & URA\\
\midrule
$\ell_2$-\ours & \textbf{97.3} & \textbf{96.2} & 96.2 & \textbf{95.6} & 95.6 & 85.8 & 85.8 & \textbf{85.8} & 85.8\\
$\ell_2$-GLVQ & 95.8 & 94.2 & 94.2 & 93.2 & 93.2 & 80.9 & 80.9 & 80.9 & 80.9\\
1-NN  & 96.9 & 95.0 & - & 93.6 & 93.6 & 78.3 & - & 78.3 & -\\
\midrule
MMR-U & 97.0 & 79.2 & 93.6 & 89.6 & 93.8 & \textbf{87.6} & 87.6 & 79.2 & 87.6\\
\bottomrule
\end{tabular}
\end{small}
\end{center}
\vskip -0.1in
\end{table}

 The advantage of our $\ell_2$-NPC is that we can certify any $\ell_q$-threat model, especially $\ell_1$ and $\ell_\infty$. This allows us to compute the \textbf{exact robust accuracy in the union of the $\ell_1$-, $\ell_2$- and $\ell_\infty$-balls.} The only other approach which has provided certified lower bounds (CRA) on multiple-norm robustness is  MMR-U from \cite{croce2020provable} who certify a neural network. In Table \ref{tab:MNIST-mult} we compare our multiple-norm robust accuracy for the $\epsilon_q$ which were chosen in \cite{croce2020provable}. 
 Our $\ell_2$-\ours outperforms MMR-U significantly in terms of certified $\ell_1$-and $\ell_2$-robustness as well as in the union.

\begin{table}[t]
\caption{\textbf{MNIST:} lower (CRA) and upper bounds (URA) on $\ell_\infty$-robust accuracy for $\ell_\infty$-NPC obtained using Lemma~\ref{lem:minmaxrho}.   }\label{tab:MNIST-linflinf}
\begin{center}
\begin{small}
\setlength{\tabcolsep}{2pt}
\begin{tabular}{l|c|cc|cc|cc|}
MNIST & std. &  \multicolumn{2}{c|}{$\epsilon_\infty=0.1$} & 
 \multicolumn{2}{c|}{$\epsilon_\infty=0.3$} & 
 \multicolumn{2}{c|}{$\epsilon_\infty=0.4$}\\
& acc. & CRA & URA & CRA & URA & CRA & URA\\
\midrule

$\ell_\infty$-\ours &94.69  &  91.19 & 91.19 &  78.68& 78.86& 65.58 & 65.96  \\
\midrule
$\ell_\infty$-GLVQ  &96.34  &  93.52 & 93.52 &  80.76& 81.04& 61.29 &62.94   \\
\midrule
$\ell_\infty$-neuron & \textbf{98.6} & - & - & \textbf{93.1} & 95.3 & - & - \\
CROWN-IBP & 98.2 & - & - & 93.0 & 94.0 & \textbf{87.4} & 90.4\\
ReLU-S & 97.3 & - & - & 80.7 & 92.1 & - & - \\
CAP & 87.4 & - & - & 56.9 & - & - & - \\
\bottomrule
\end{tabular}
\end{small}
\end{center}
\vskip -0.1in
\end{table}

\textbf{MNIST - $\ell_\infty$-NPC}
We compare our $\ell_\infty$-\ours to the $\ell_\infty$-GLVQ of \cite{Saralajew2020fast}. For reference we provide the best results for the $\ell_\infty$-certfied neural networks: $\ell_\infty$-neurons \cite{pmlr-v139-zhang21b}, CROWN-IBP \cite{zhang2019stable}, as well as slightly older results; ReLU-stability \cite{xiao2019training} and CAP \cite{WonEtAl18} to put our results into context. We perform slightly worse than \cite{Saralajew2020fast} for small radii, but significantly better for the bigger one. 
Due to our better lower bounds but also by using AutoAttack \cite{croce2020reliable} for computing the upper bounds we close the gap between upper and lower bounds from $4.2\%$ in \cite{Saralajew2020fast} to $0.3\%$. \vv{To attack the classifier with AutoAttack, we interpret the negative distance to the closest prototype from a particular class as the logit value.}

\paragraph{Perceptual metric NPC}
As discussed in Section \ref{sec:perceptual} it is unlikely that $\ell_p$-NPC will work for image classifcation tasks like CIFAR10.
However, with the perceptual metric LPIPS (based on Alexnet) which corresponds to an $\ell_2$-metric in the embedding space, we get much better results with our Perceptual-\ours (P-\ours).
 In Figure \ref{fig:rcurve-lpips} we show the certified robust accuracy (lower bound of Lemma~\ref{lem:minmaxrho}) as a function of the LPIPS-radius for the standard $\ell_2$-lower bounds and for the improved lower bounds taking into account the constraints of the embedding. We have three important observations. We achieve a clean accuracy of 80.3$\%$ which is quite remarkable for a classifier with certified robust accuracy. Second, this is up to our knowledge the first result on certified robustness with respect to the LPIPS-threat model. Third, \cite{laidlaw2021perceptual} who do empirical perceptual adversarial training with a a ResNet 50 get only
 $71.6\%$ clean accuracy and only a URA of $9.8\%$ which is more than $30\%$ \textbf{worse} than our CRA of $40.5\%$. 
Moreover, our URA computed using the LPA-attack of \cite{laidlaw2021perceptual} is with $70.3\%$ remarkably high. These are very promising results justifying more research in \ours for perceptual metrics.

\vv{On the other hand, in~\cite{laidlaw2021perceptual} it is noted that models trained to be robust w.r.t. LPIPS-threat model  are empirically robust also to other threat models such as $\ell_2$ or $\ell_\infty$ - even though one has to state that their model has only a robust accuracy of 9.8\%. This generalization does not hold for P-PNPC. For $\ell_\infty$ threat model, we observed (empirical) robust accuracies $49\%$, $23\%$, $2\%$, $0\%$ for radii $1/255, 2/255, 4/255, 8/255$. For $\ell_2$ we have robust accuracy $51\%, 29\%, 5\%, 0\%$ for radii $0.14, 0.25, 0.5, 1$. While the robust accuracies are non-trivial, they are not comparable to the ones achieved in~\cite{laidlaw2021perceptual}. As our P-PNPC is much more robust with respect to the LPIPS-threat model than the neural network of \cite{laidlaw2021perceptual}, it is thus an open question if this threat model leads indeed to a generalization to other threat models.}

\begin{figure}[ht]
\begin{center}
\includegraphics[width=0.99\columnwidth]{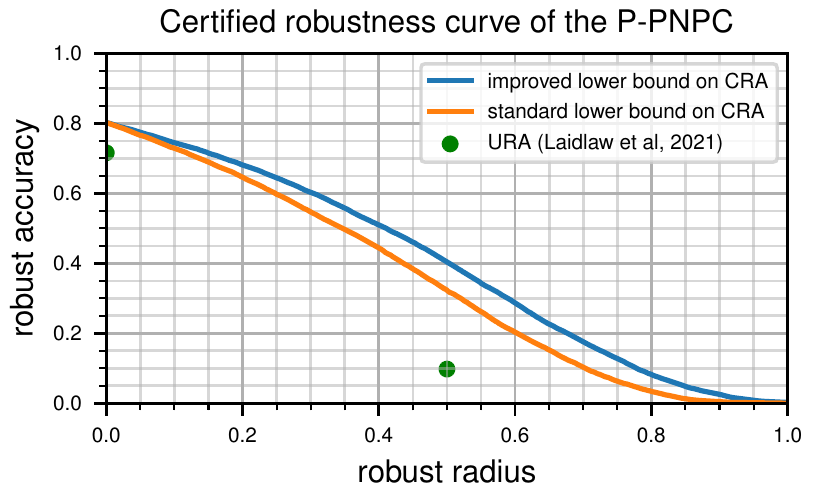}
\caption{\label{fig:rcurve-lpips}The certified robust accuracy as a function of the radius of the LPIPS-threat model. Integrating the spherical plus non-negativity constraints leads to huge improvements. The standard accuracy as well as the \textbf{empirical} robust accuracy of \cite{laidlaw2021perceptual} are \emph{worse} than  \textbf{certified} robust accuracy of P-\ours by a large margin.}
\end{center}
\end{figure}
\section{Conclusion}
We have provided theoretical foundations as well as  efficient algorithmic tools for the computation of the exact minimal adversarial perturbation, as well as lower bounds, for nearest prototype classifiers for several threat models, including the perceptual metric LPIPS. We have shown SOTA performance for deterministic $\ell_2$-certification on MNIST and remarkably strong certified robustness results with respect to the LPIPS metric. Thus we think that NPC deserve more attention in our research community.

\section*{Acknowledgements}
\vv{The authors acknowledge support from the German Federal
Ministry of Education and Research (BMBF) through the
T{\"u}bingen AI Center (FKZ: 01IS18039A) and the DFG Cluster of Excellence ``Machine Learning – New Perspectives for
Science”, EXC 2064/1, project number 390727645.}
The authors are thankful for the support of Open Philanthropy.

\bibliography{main}

\begin{thebibliography}{42}
\providecommand{\natexlab}[1]{#1}
\providecommand{\url}[1]{\texttt{#1}}
\expandafter\ifx\csname urlstyle\endcsname\relax
  \providecommand{\doi}[1]{doi: #1}\else
  \providecommand{\doi}{doi: \begingroup \urlstyle{rm}\Url}\fi

\bibitem[Andriushchenko \& Hein(2019)Andriushchenko and
  Hein]{andriushchenko2019provably}
Andriushchenko, M. and Hein, M.
\newblock Provably robust boosted decision stumps and trees against adversarial
  attacks.
\newblock In \emph{NeurIPS}, 2019.

\bibitem[Athalye et~al.(2018)Athalye, Carlini, and Wagner]{AthEtAl2018}
Athalye, A., Carlini, N., and Wagner, D.~A.
\newblock Obfuscated gradients give a false sense of security: Circumventing
  defenses to adversarial examples.
\newblock In \emph{ICML}, 2018.

\bibitem[Bertsimas et~al.(2018)Bertsimas, Dunn, Pawlowski, and
  Zhuo]{bertsimas2018robust}
Bertsimas, D., Dunn, J., Pawlowski, C., and Zhuo, Y.~D.
\newblock Robust classification.
\newblock \emph{INFORMS Journal on Optimization}, 1:\penalty0 2--34, 2018.

\bibitem[Biggio et~al.(2013)Biggio, Corona, Maiorca, Nelson, {\v{S}}rndi{\'c},
  Laskov, Giacinto, and Roli]{biggio2013evasion}
Biggio, B., Corona, I., Maiorca, D., Nelson, B., {\v{S}}rndi{\'c}, N., Laskov,
  P., Giacinto, G., and Roli, F.
\newblock Evasion attacks against machine learning at test time.
\newblock In \emph{Joint European conference on machine learning and knowledge
  discovery in databases}, pp.\  387--402. Springer, 2013.

\bibitem[Carlini et~al.(2019)Carlini, Athalye, Papernot, Brendel, Rauber,
  Tsipras, Goodfellow, Madry, and Kurakin]{carlini2019evaluating}
Carlini, N., Athalye, A., Papernot, N., Brendel, W., Rauber, J., Tsipras, D.,
  Goodfellow, I., Madry, A., and Kurakin, A.
\newblock On evaluating adversarial robustness.
\newblock \emph{arXiv preprint arXiv:1902.06705}, 2019.

\bibitem[Chen et~al.(2019)Chen, Zhang, Boning, and Hsieh]{chen2019robust}
Chen, H., Zhang, H., Boning, D., and Hsieh, C.-J.
\newblock Robust decision trees against adversarial examples.
\newblock In \emph{ICML}, 2019.

\bibitem[Cohen et~al.(2019)Cohen, Rosenfeld, and Kolter]{CohenARXIV2019}
Cohen, J.~M., Rosenfeld, E., and Kolter, J.~Z.
\newblock Certified adversarial robustness via randomized smoothing.
\newblock In \emph{NeurIPS}, 2019.

\bibitem[Crammer et~al.(2003)Crammer, Gilad-bachrach, Navot, and
  Tishby]{NIPS2002_bbaa9d6a}
Crammer, K., Gilad-bachrach, R., Navot, A., and Tishby, N.
\newblock Margin analysis of the lvq algorithm.
\newblock In \emph{NeurIPS}, 2003.

\bibitem[Croce \& Hein(2020{\natexlab{a}})Croce and Hein]{croce2020provable}
Croce, F. and Hein, M.
\newblock Provable robustness against all adversarial $l_p$-perturbations for
  $p\geq 1$.
\newblock In \emph{ICLR}, 2020{\natexlab{a}}.

\bibitem[Croce \& Hein(2020{\natexlab{b}})Croce and Hein]{croce2020reliable}
Croce, F. and Hein, M.
\newblock Reliable evaluation of adversarial robustness with an ensemble of
  diverse parameter-free attacks, 2020{\natexlab{b}}.

\bibitem[Goodfellow et~al.(2015)Goodfellow, Shlens, and Szegedy]{GooShlSze2015}
Goodfellow, I.~J., Shlens, J., and Szegedy, C.
\newblock Explaining and harnessing adversarial examples.
\newblock In \emph{ICLR}, 2015.

\bibitem[Gowal et~al.(2018)Gowal, Dvijotham, Stanforth, Bunel, Qin, Uesato,
  Arandjelovic, Mann, and Kohli]{GowEtAl18}
Gowal, S., Dvijotham, K., Stanforth, R., Bunel, R., Qin, C., Uesato, J.,
  Arandjelovic, R., Mann, T.~A., and Kohli, P.
\newblock On the effectiveness of interval bound propagation for training
  verifiably robust models.
\newblock preprint, arXiv:1810.12715v3, 2018.

\bibitem[Hein \& Andriushchenko(2017)Hein and Andriushchenko]{HeiAnd2017}
Hein, M. and Andriushchenko, M.
\newblock Formal guarantees on the robustness of a classifier against
  adversarial manipulation.
\newblock In \emph{NeurIPS}, 2017.

\bibitem[Huang et~al.(2021)Huang, Zhang, Shi, Kolter, and
  Anandkumar]{huang2021training}
Huang, Y., Zhang, H., Shi, Y., Kolter, J.~Z., and Anandkumar, A.
\newblock Training certifiably robust neural networks with efficient local
  lipschitz bounds.
\newblock In \emph{NeurIPS}, 2021.

\bibitem[Jeong et~al.(2021)Jeong, Park, Kim, Lee, Kim, and
  Shin]{jeong2021smoothmix}
Jeong, J., Park, S., Kim, M., Lee, H.-C., Kim, D., and Shin, J.
\newblock Smoothmix: Training confidence-calibrated smoothed classifiers for
  certified robustness.
\newblock In \emph{NeurIPS}, 2021.

\bibitem[Kantchelian et~al.(2016)Kantchelian, Tygar, and
  Joseph]{kantchelian2016evasion}
Kantchelian, A., Tygar, J., and Joseph, A.
\newblock Evasion and hardening of tree ensemble classifiers.
\newblock In \emph{ICML}, 2016.

\bibitem[Kireev et~al.(2021)Kireev, Andriushchenko, and
  Flammarion]{kireev2022effectiveness}
Kireev, K., Andriushchenko, M., and Flammarion, N.
\newblock On the effectiveness of adversarial training against common
  corruptions.
\newblock \emph{arXiv preprint, arXiv:2103.02325}, 2021.

\bibitem[Kohonen(1995)]{Kohonen1995}
Kohonen, T.
\newblock \emph{Learning Vector Quantization}, pp.\  175--189.
\newblock Springer Berlin Heidelberg, 1995.

\bibitem[Laidlaw et~al.(2021)Laidlaw, Singla, and Feizi]{laidlaw2021perceptual}
Laidlaw, C., Singla, S., and Feizi, S.
\newblock Perceptual adversarial robustness: Defense against unseen threat
  models.
\newblock In \emph{ICLR}, 2021.

\bibitem[Lee et~al.(2020)Lee, Lee, and Park]{Lee2020Lipschitz}
Lee, S., Lee, J., and Park, S.
\newblock Lipschitz-certifiable training with a tight outer bound.
\newblock In \emph{NeurIPS}, 2020.

\bibitem[Leino et~al.(2021)Leino, Wang, and
  Fredrikson]{leino2021globallyrobust}
Leino, K., Wang, Z., and Fredrikson, M.
\newblock Globally-robust neural networks.
\newblock In \emph{ICML}, 2021.

\bibitem[Li et~al.(2020)Li, Qi, Xie, and Li]{li2020sokcertified}
Li, L., Qi, X., Xie, T., and Li, B.
\newblock Sok: Certified robustness for deep neural networks.
\newblock \emph{arXiv preprint arXiv:2009.04131}, 2020.

\bibitem[Li et~al.(2019)Li, Haque, Anil, Lucas, Grosse, and
  Jacobsen]{li2019preventing}
Li, Q., Haque, S., Anil, C., Lucas, J., Grosse, R.~B., and Jacobsen, J.-H.
\newblock Preventing gradient attenuation in lipschitz constrained
  convolutional networks.
\newblock In \emph{NeurIPS}, 2019.

\bibitem[Mirman et~al.(2018)Mirman, Gehr, and Vechev]{MirGehVec2018}
Mirman, M., Gehr, T., and Vechev, M.
\newblock Differentiable abstract interpretation for provably robust neural
  networks.
\newblock In \emph{ICML}, 2018.

\bibitem[Papernot et~al.(2016)Papernot, McDaniel, and
  Goodfellow]{papernot2016transferability}
Papernot, N., McDaniel, P., and Goodfellow, I.
\newblock Transferability in machine learning: from phenomena to black-box
  attacks using adversarial samples, 2016.

\bibitem[Russu et~al.(2016)Russu, Demontis, Biggio, Fumera, and
  Roli]{russu2016secure}
Russu, P., Demontis, A., Biggio, B., Fumera, G., and Roli, F.
\newblock Secure kernel machines against evasion attacks.
\newblock In \emph{ACM workshop on AI and security}. ACM, 2016.

\bibitem[Saralajew et~al.(2020)Saralajew, Holdijk, and
  Villmann]{Saralajew2020fast}
Saralajew, S., Holdijk, L., and Villmann, T.
\newblock Fast adversarial robustness certification of nearest prototype
  classifiers for arbitrary seminorms.
\newblock In \emph{NeurIPS}, 2020.

\bibitem[Singla et~al.(2022)Singla, Singla, and Feizi]{singla2022improved}
Singla, S., Singla, S., and Feizi, S.
\newblock Improved deterministic l2 robustness on {CIFAR}-10 and {CIFAR}-100.
\newblock In \emph{ICLR}, 2022.

\bibitem[Szegedy et~al.(2014)Szegedy, Zaremba, Sutskever, Bruna, Erhan,
  Goodfellow, and Fergus]{SzeEtAl2014}
Szegedy, C., Zaremba, W., Sutskever, I., Bruna, J., Erhan, D., Goodfellow, I.,
  and Fergus, R.
\newblock Intriguing properties of neural networks.
\newblock In \emph{ICLR}, pp.\  2503--2511, 2014.

\bibitem[Tjeng \& Tedrake(2017)Tjeng and Tedrake]{TjeTed2017}
Tjeng, V. and Tedrake, R.
\newblock Verifying neural networks with mixed integer programming.
\newblock preprint, arXiv:1711.07356v1, 2017.

\bibitem[Tramer et~al.(2020)Tramer, Carlini, Brendel, and
  Madry]{tramer2020adaptive}
Tramer, F., Carlini, N., Brendel, W., and Madry, A.
\newblock On adaptive attacks to adversarial example defenses.
\newblock In \emph{NeurIPS}, 2020.

\bibitem[Trockman \& Kolter(2021)Trockman and
  Kolter]{trockman2021orthogonalizing}
Trockman, A. and Kolter, J.~Z.
\newblock Orthogonalizing convolutional layers with the cayley transform.
\newblock In \emph{ICLR}, 2021.

\bibitem[Wang et~al.(2019)Wang, Liu, Yi, Zhou, and Hsieh]{wang2019evaluating}
Wang, L., Liu, X., Yi, J., Zhou, Z.-H., and Hsieh, C.-J.
\newblock Evaluating the robustness of nearest neighbor classifiers: A
  primal-dual perspective.
\newblock \emph{arXiv preprint, arXiv:1906.03972}, 2019.

\bibitem[Wang et~al.(2018)Wang, Jha, and Chaudhuri]{wang2018analyzing}
Wang, Y., Jha, S., and Chaudhuri, K.
\newblock Analyzing the robustness of nearest neighbors to adversarial
  examples.
\newblock In \emph{ICML}, 2018.

\bibitem[Wang et~al.(2004)Wang, Bovik, Sheikh, and Simoncelli]{wang2004image}
Wang, Z., Bovik, A.~C., Sheikh, H.~R., and Simoncelli, E.~P.
\newblock Image quality assessment: from error visibility to structural
  similarity.
\newblock \emph{IEEE transactions on image processing}, 13\penalty0
  (4):\penalty0 600--612, 2004.

\bibitem[Wong \& Kolter(2018)Wong and Kolter]{WonKol2018}
Wong, E. and Kolter, J.~Z.
\newblock Provable defenses against adversarial examples via the convex outer
  adversarial polytope.
\newblock In \emph{ICML}, 2018.

\bibitem[Wong et~al.(2018)Wong, Schmidt, Metzen, and Kolter]{WonEtAl18}
Wong, E., Schmidt, F., Metzen, J.~H., and Kolter, J.~Z.
\newblock Scaling provable adversarial defenses.
\newblock In \emph{NeurIPS}, 2018.

\bibitem[Xiao et~al.(2019)Xiao, Tjeng, Shafiullah, and Madry]{xiao2019training}
Xiao, K.~Y., Tjeng, V., Shafiullah, N.~M., and Madry, A.
\newblock Training for faster adversarial robustness verification via inducing
  relu stability.
\newblock In \emph{ICLR}, 2019.

\bibitem[Xu et~al.(2009)Xu, Caramanis, and Mannor]{XuCarMan2009}
Xu, H., Caramanis, C., and Mannor, S.
\newblock Robustness and regularization of support vector machines.
\newblock \emph{Journal of Machine Learning Research}, 10:\penalty0 1485--1510,
  2009.

\bibitem[Zhang et~al.(2021)Zhang, Cai, Lu, He, and Wang]{pmlr-v139-zhang21b}
Zhang, B., Cai, T., Lu, Z., He, D., and Wang, L.
\newblock Towards certifying l-infinity robustness using neural networks with
  l-inf-dist neurons.
\newblock In \emph{ICML}, 2021.

\bibitem[Zhang et~al.(2020)Zhang, Chen, Xiao, Gowal, Stanforth, Li, Boning, and
  Hsieh]{zhang2019stable}
Zhang, H., Chen, H., Xiao, C., Gowal, S., Stanforth, R., Li, B., Boning, D.,
  and Hsieh, C.-J.
\newblock Towards stable and efficient training of verifiably robust neural
  networks.
\newblock In \emph{ICLR}, 2020.

\bibitem[Zhang et~al.(2018)Zhang, Isola, Efros, Shechtman, and
  Wang]{zhang2018unreasonable}
Zhang, R., Isola, P., Efros, A.~A., Shechtman, E., and Wang, O.
\newblock The unreasonable effectiveness of deep features as a perceptual
  metric.
\newblock In \emph{CVPR}, 2018.

\end{thebibliography}
\bibliographystyle{icml2022}

\newpage
\appendix
\onecolumn
The appendix includes the missing proofs from the paper (App. \ref{app:proof-trivlb} to App. \ref{app:lpips}), results for $\ell_2$-NPC for CIFAR10 in App. \ref{app:cifar} and comparison to orthogonal convolutions in~\ref{app:SOC}.
\section{Proof of Theorem \ref{thm:trivlb}}\label{app:proof-trivlb}
\begin{proof}
We note that for any $x$ it holds by the triangle inequality
\[ d(x,w_i) \leq d(z,x) + d(w_i,z).\]
Thus it holds
\[ d(x,w_i) - d(x,w_j) \leq d(z,w_i) + d(x,z) - d(z,w_j) + d(x,z),\]
and we get that all points in $B_d(z,r)$ are classified the same as $z$ if
\begin{align*}
 \maxop_{x \in B_d(z,r)} \Big(\minop_{i \in I_y}  d(x,w_i) -  \minop_{j \in I^c_y}d(x,w_j)\Big)\leq  \minop_{i \in I_y}  d(z,w_i) - \minop_{j \in I^c_y}d(z,w_j) + 2r \leq 0\\ 
 \end{align*}
 This yields that
 \[ r \leq \frac{\minop_{j \in I^c_y} d(z,w_j)-\minop_{i \in I_c} d(z,w_i)}{2}.\]
\end{proof}


\section{Proof of Theorem \ref{th:l2-certificate}}
\begin{proof}
We define the set
\[ U^{(p)}_j=\{ x \in \R^n \,|\, \norm{x-w_i}_p - \norm{x-w_j}_p \geq 0 \quad \forall \, i \in I_y\}.\]
as the set of points which are not classified as $y$ when only the single prototype with index $j \in I_y^c$ would be considered. 
We get the full set of points not classified as $y$ as
the union $\bigcup_{j \in I_y^c} U^{(p)}_j$. 
We define $r^q_p(z)_j=\min_{x \in U^{(p)}_j} \norm{z-x}_q$ as the radius of the largest $\ell_q$-ball which still fits
into $\R^d \backslash U^{(p)}_j$ and thus is fully classified as class $y$ when only considering $j \in I_y^c$. Thus the radius $\epsilon^q_p(z)$ of the largest $\ell_q$-ball fitting into
$\R^d \backslash \bigcup_{j \in I_y^c} U^{(p)}_j = \bigcap_{j \in I_y^c} \Big( \R^d \backslash U^{(p)}_j\Big)$ is given by
\[\epsilon^q_p(z)=\minop_{j \in I_y^c} r^q_p(z)_j,\]
which can be seen using the fact that $r^q_p(z)_j$ is the minimal $\ell_q$-distance to $U_j^{(p)}$.
\end{proof}

\section{Proof of Lemma \ref{lem:minmaxrho}}

\begin{proof}
As for each $i \in I_y$ the problem for $\rho^q_p(z)_{i,j}$ is a relaxation of the problem for $r^q_p(z)_j$ (as we are omitting constraints), it holds for each $i \in I_y$:
\[ r^q_p(z)_j \geq \rho^q_p(z)_{i,j} \;\Longrightarrow\; r^q_p(z)_j \geq \maxop_{i \in I_y}\rho^q_p(z)_{i,j}.\]
Thus 
\[ \epsilon^p_q(z)=\min_{j \in I^c_y} r^q_p(z)_j \; \geq \; \minop_{j \in I^y_c}\maxop_{i \in I_y}\rho^q_p(z)_{i,j}.\]

For the second part if $x^*$ satisfies 
\[ \norm{x^*-w_i}_p - \norm{x^*-w_{j^*}}_p \geq 0 \quad \forall i \in I_y,\]
then it is a feasible point for the optimization problem of $r^q_p(z)_{j^*}$ in \eqref{eq:rqp} and thus $\rho^q_p(z)_{i^*,j^*}=r^q_p(z)_{j^*}$. By definition and by the just derived result it holds
\[ \rho^q_p(z)_{i^*,j^*}=r^q_p(z)_{j^*}\geq \epsilon^q_p(z) \geq \rho^q_p(z)_{i^*,j^*},\]
and thus equality has to hold.
\end{proof}
\section{Proof of Theorem \ref{th:betterbound}}

\begin{lemma}\label{lem:trivos}
$\rho_p^p(z)_{i,j} \geq \frac{\norm{w_j-z}_p - \norm{w_i-z}_p}{2}$
\end{lemma}
\begin{proof}
First we restate the definition of $\rho$:
 
\begin{align}
 \rho^q_p(z)_{i,j}=\minop_{x \in \R^d} & \quad\norm{x-z}_q \\
	 \textrm{sbj. to:} & \quad \norm{x-w_i}_p - \norm{x-w_j}_p \geq 0 \nonumber\\
	                   & x \in \X \nonumber
\end{align}

We consider $p=q$. By the triangle inequality the following holds for any $x \in \X$, thus also for any adversarial perturbation $x$ for which $\norm{x-w_i}_p - \norm{x-w_j}_p \geq 0$:
\begin{equation}
    \begin{aligned}
\norm{x-w_i}_p \leq \norm{x-z}_p + \norm{z-w_i}_p &\\
\norm{z-w_j}_p \leq \norm{x-z}_p + \norm{x-w_j}_p  & \quad \Longrightarrow \quad \norm{x-w_j}_p \geq \norm{z-w_j}_p - \norm{x-z}_p
\end{aligned}
\end{equation} 
Summing the inequalities up we get for any feasible $x \in \X$ satisfying the inequality constraint, \begin{equation}
\begin{aligned}
    \norm{z-w_i}_p - \norm{z-w_j} + 2\norm{x-z}_p &\geq \norm{x-w_i}_p - \norm{x-w_j}_p \geq 0.
 \end{aligned}
 \end{equation}
 which yields finally
 \begin{equation}
 \begin{aligned}
    \norm{x-z}_p \geq \frac{\norm{w_j-z}_p - \norm{w_i-z}_p}{2}.
\end{aligned}
\end{equation}
Therefore, $\rho_p^p(z)_{i,j} \geq \frac{\norm{w_j-z}_p - \norm{w_i-z}_p}{2}$.
\end{proof}

\begin{proof}[Proof of Theorem~\ref{th:betterbound}]
If $z$ is misclassified, then it reduces to $0 \geq 0$ which holds. Otherwise,
by Lemma~\ref{lem:trivos}, it holds $ \rho^p_p(z)_{i,j} \geq \frac{\norm{z-w_j}_p - \norm{z-w_i}_p}{2}$.
Then 
\begin{align*}
     \minop_{j \in I^y_c}\maxop_{i \in I_y} \rho^p_p(z)_{i,j} &\geq \minop_{j \in I^y_c} \rho^p_p(z)_{i^*, j}\\ &\geq \minop_{j \in I^y_c} \frac{\norm{z-w_j}_p - \norm{z-w_{i^*}}_p}{2} \\&= \frac{\minop_{j \in I_y^c} \norm{z-w_j}_p-\minop_{i \in I_y}\norm{z-w_i}_p}{2}
\end{align*}

We further show that there are cases where the inequality is strict. Consider a $d$-dimensional example where $z=(0,\dots,0)$, $\{w_j \mid j \in I^y_c\} = \{(2,0,\dots,0)\}$, $\{w_i \mid  i \in I^y\} = \{(1,0,\dots,0)\}$. It clearly holds that $\minop_{j \in I^y_c}\maxop_{i \in I_y} \rho^p_p(z)_{i,j} = 1.5$, while $\max\left\{0, \frac{\minop_{j \in I_y^c} \norm{z-w_j}_p-\minop_{i \in I_y}\norm{z-w_i}_p}{2}\right\} = 1$ for any $p$.

\end{proof}


\section{Proof of Theorem~\ref{thm:rhocomplex}}

\textbf{Theorem~\ref{thm:rhocomplex}}\emph{
The computational complexities of optimization problems $\rho_p^q(z)_{i,j}$ for $p,q \in \{1,2,\infty\}$ for $\X=\R^d$ are summarised in Table~\ref{tab:hardness_rho_app}.}
\begin{table}
\begin{center}
\begin{tabular}{c|c| c | c| c|} 
 & & \multicolumn{3}{c|}{$\ell_q$-threat model}\\
 \cline{2-5}
\multirow{4}{*}{ \begin{turn}{90}\hspace{+1mm} $\ell_p$-distance \end{turn}}& & $\ell_1$ & $\ell_2$ & $\ell_\infty$\\ 
\cline{2-5}
 & $\ell_1$ & NP-hard & NP-hard & $O(d \log(d))$ \\ 
 \cline{2-5}
  & $\ell_2$ & $\Theta(d)$ & $\Theta(d)$ & $\Theta(d)$ \\
 \cline{2-5}
 & $\ell_\infty$ & $\Theta(d)$ &  $O(d \log(d))$& $\Theta(d)$ \\
\cline{2-5}
\end{tabular}
\end{center}
\caption{
Computational complexity of $\rho_p^q(z)_{i,j}$.} \label{tab:hardness_rho_app} 
\end{table}

Throughout the proof, we assume $z$ is correctly classified, otherwise the solution is $0$.
We prove the theorem gradually for cases $p=2$ and any $q$, then $q=\infty$ and any $p$, then $p=1$ and any $q\neq \infty$ and finally $p=\infty$, $q=1,2$. For most of the cases,  we discuss the possibility of incorporating box constraints, which usually increases complexity from $O(d)$ to $O(d\log(d))$. We also remark that using the median of medians algorithm, one could avoid sorting coordinates, and could achieve $\Theta(d)$ complexities. We, for the sake of simplicity, will be sorting point for the price of $log(d)$ factor in complexity.

\begin{proof}[\textbf{Proof for case $p=2$ and any $q$}]
\begin{align}
 \rho^q_2(z)_{i,j}=\minop_{x \in \R^d} & \quad\norm{x-z}_q \\
	 \textrm{sbj. to:} & \quad \norm{x-w_i}_2 - \norm{x-w_j}_2 \geq 0 \nonumber\\
	                   & x \in \X \nonumber
\end{align}
We equivalently rewrite the constraint in the following way:

\begin{equation*}
    \begin{aligned}
     \norm{x-w_i}_2 - \norm{x-w_j}_2 &\geq 0, \\
     \norm{x - z + z -w_i}^2_2 - \norm{x -z + z-w_j}^2_2 &\geq 0,\\
     \norm{x-z}_2^2 + 2\inner{x-z, z-w_i} + \norm{z-w_i}^2_2 - \left(\norm{x-z}_2^2 + 2\inner{x-z, z-w_j} + \norm{z-w_j}^2_2\right) &\geq 0,\\
     2\inner{x-z, w_j - w_i} &\geq \norm{z - w_j}_2^2 - \norm{z-w_i}_2^2, \\
     2\norm{x-z}_q\norm{w_j-w_i}_{\frac{q}{q-1}} \geq 2\inner{x-z, w_j - w_i} &\geq \norm{z - w_j}_2^2 - \norm{z-w_i}_2^2, \\
    \norm{x-z}_q &\geq \frac{\norm{z - w_j}_2^2 - \norm{z-w_i}_2^2}{2\norm{w_j-w_i}_{\frac{q}{q-1}}}.
    \end{aligned}
\end{equation*}
Since Hölder's inequality is tight, we conclude  

\begin{equation}\label{eq:l2-lb}
    \rho^q_2(z)_{i,j}= \frac{\norm{z - w_j}_2^2 - \norm{z-w_i}_2^2}{2\norm{w_j-w_i}_{\frac{q}{q-1}}}.
\end{equation}
We note that analogical derivation holds for minimising $\norm{x-z}$ in any norm, not just for the $q$-norm. In that case, $\norm{\cdot}_{\frac{q}{q-1}}$ is replaced with the dual norm of the considered norm. 
The box-constrained version of this problem can be solved
in $O(d \log d)$, see e.g., Section $4$ of ~\cite{HeiAnd2017}.
\end{proof}

\begin{proof}[\textbf{Proof for case $p=q=\infty$}]

\begin{align}\label{prob:rhoinfty}
 \rho^\infty_p(z)_{i,j}=\minop_{x \in \R^d} & \quad\norm{x-z}_\infty \\
	 \textrm{sbj. to:} & \quad \norm{x-w_i}_\infty - \norm{x-w_j}_\infty \geq 0 \nonumber\\
	                   & x \in \X \nonumber
\end{align}

We note that whenever $\norm{x-w_i}_\infty - \norm{x-w_j}_\infty \geq 0$, then also $\norm{x'-w_i}_\infty - \norm{x'-w_j}_\infty \geq 0$, where $x'^{(l)} = x^{(l)}+\alpha\sign(w_j^{(l)} - w_i^{(l)})$ for any positive $\alpha$ and some $l = 1 \dots d$, and $x'^{(l)} = x^{(l)}$ for the coordinates. That is, we can move $x^{(l)}$ in the direction from $w_i^{(l)}$ to $w_j^{(l)}$, since if $|x'^{(l)}-w_j^{(l)}| > |x^{(l)}-w_j{(l)}| $, then also  $|x'^{(l)} - w_i^{l}| >|x'^{(l)}-w_j{(l)}|$. On the other hand, if $|x^{(l)}-w_i^{(l)}| > |x'^{(l)} -w_i^{(l)}|$, then  also $|x^{(l)}-w_i^{(l)}| < |x^{(l)} - w_j^{(l)}|$, thus $l$ was not the maximising index of $\norm{x-w_i}_\infty$, and consequently $\norm{x-w_i}_\infty = \norm{x'-w_i}_\infty$. The remaining case is trivial; thus, $\norm{x'-w_i}_p - \norm{x'-w_j}_p \geq 0$. This argument may be repeated $d$ times to conclude that when $\rho_\infty^\infty(z)_{i,j} = \epsilon$, then a minimizer of Problem \ref{prob:rhoinfty} is $x^* = z + \epsilon\sign{(w_j - w_i)}$. 

Therefore, the problem is to find the smallest $\epsilon$ for which $\norm{z+\epsilon \sign(w_j-w_i)-w_i}_\infty \geq \norm{z+\epsilon\sign(w_j-w_i)-w_j}_\infty$.  Note that \[\norm{z\!+\!\epsilon\sign(w_j\!-w_i\!)\!-\!w_j}_\infty = \maxop_{l= 1,\dots,d}{\max\left\{z^{(l)}+\epsilon\sign\left(w_j^{(l)}-w_i^{(l)}\right)-w_j^{(l)},\!-\!\left(z^{(l)}+\epsilon\sign\left(w_j^{(l)}-w_i^{(l)}\right)-w_j^{(l)}\right)\!\right\}};
\]thus, it is a maximum of $2d$ linear functions, each of which has slope either $1$, or $-1$. Let $\alpha_i = \argmin{\sign(w_j-w_i)(z-w_i)} $ and $\beta_i = \argmax{\sign(w_j-w_i)(z-w_i)} $, analogously for $\alpha_j$, $\beta_j$. Then 

\begin{align*}
    &\norm{z+\epsilon \sign(w_j-w_i)-w_i}_\infty -\norm{z+\epsilon\sign(w_j-w_i)-w_j}_\infty = \\
    &\max\left\{-\epsilon -  \sign\left(w_j^{(\alpha_i)}-w_i^{(\alpha_i)}\right) \left(z^{(\alpha_i)} - w_i^{(\alpha_i)}\right),  \epsilon+  \sign\left(w_j^{(\beta_i)}-w_i^{(\beta_i)}\right) \left(z^{(\beta_i)} - w_i^{(\beta_i)}\right)\right\} - \\
        &\max\left\{-\epsilon -  \sign\left(w_j^{(\alpha_j)}-w_i^{(\alpha_j)}\right) \left(z^{(\alpha_j)} - w_j^{(\alpha_j)}\right),  \epsilon + \sign\left(w_j^{(\beta_j)}-w_i^{(\beta_j)}\right) \left(z^{(\beta_j)} - w_j^{(\beta_j)}\right)\right\}.
\end{align*}
Moreover, we can analyse to slope of $\norm{z+\epsilon \sign(w_j-w_i)-w_i}_\infty -\norm{z+\epsilon\sign(w_j-w_i)-w_j}_\infty$ and see that it is non-zero only in the interval between points
\[ \epsilon_i = \frac{-\sign\left(w_j^{(\alpha_i)}-w_i^{(\alpha_i)}\right) \left(z^{(\alpha_i)} - w_i^{(\alpha_i)}\right) - \sign\left(w_j^{(\beta_i)}-w_i^{(\beta_i)}\right) \left(z^{(\beta_i)} - w_i^{(\beta_i)}\right))}{2},\]
and 

\[ \epsilon_j = \frac{-\sign\left(w_j^{(\alpha_j)}-w_i^{(\alpha_j)}\right) \left(z^{(\alpha_j)} - w_j^{(\alpha_j)}\right) - \sign\left(w_j^{(\beta_j)}-w_i^{(\beta_j)}\right) \left(z^{(\beta_j)} - w_j^{(\beta_j)}\right))}{2},\]
where the slope is $2$. Now it is easy to compute the value of $\norm{z+\epsilon \sign(w_j-w_i)-w_i}_\infty -\norm{z+\epsilon\sign(w_j-w_i)-w_j}_\infty$ for very big ($V_{+}$) and very small ($V_{-}$) values of $\epsilon$, where the active linear function is the one with negative slope. Concretely

\[ 
V_{-} =  \sign\left(w_j^{(\alpha_j)}-w_i^{(\alpha_j)}\right) \left(z^{(\alpha_j)} - w_j^{(\alpha_j)}\right) -  \sign\left(w_j^{(\alpha_i)}-w_i^{(\alpha_i)}\right) \left(z^{(\alpha_i)} - w_i^{(\alpha_i)}\right),
\]
\[ 
V_{+} =  \sign\left(w_j^{(\beta_i)}-w_i^{(\beta_i)}\right) \left(z^{(\beta_i)} - w_i^{(\beta_i)}\right) - \sign\left(w_j^{(\beta_j)}-w_i^{(\beta_j)}\right) \left(z^{(\beta_j)} - w_j^{(\beta_j)}\right).
\]

Now we use the fact that the slope is $2$ between $\epsilon_i$ and $\epsilon_j$ to find the point where the norms are equal; Thus, we can express $\rho_\infty^\infty(z)_{i,j}$ as 

\[\rho_\infty^\infty(z)_{i,j} = \max\{\epsilon_i, \epsilon_j\} - \frac{V_{+}}{2}, \]

or as 

\[\rho_\infty^\infty(z)_{i,j} = \min\{\epsilon_i, \epsilon_j\} - \frac{V_{-}}{2}. \]

We can take the mean of both expression, then we arrive at
\[\rho_\infty^\infty(z)_{i,j} = \epsilon_i + \epsilon_j - \frac{V_{-} + V_{+}}{2}, \] which simplifies to

\[\rho_\infty^\infty(z)_{i,j} =-\frac{\sign\left(w_j^{(\alpha_j)}-w_i^{(\alpha_j)}\right) \left(z^{(\alpha_j)} - w_j^{(\alpha_j)}\right) +\sign\left(w_j^{(\beta_i)}-w_i^{(\beta_i)}\right) \left(z^{(\beta_i)} - w_i^{(\beta_i)}\right) }{2},\]

and by substituting back the definitions of $\alpha_j, \beta_i$:

\begin{equation}\label{eq:linf-linf}
\rho_\infty^\infty(z)_{i,j} = \frac{\maxop_{l = 1,\dots,d}{-\sign\left(w_j^{(l)} - w_i^{(l)}\right)\left(z^{(l)} - w_j^{(l)}\right)} - \maxop_{l = 1,\dots,d}{\sign\left(w_j^{(l)} - w_i^{(l)}\right)\left(z^{(l)} - w_i^{(l)}\right)} }{2}.
\end{equation}\

\end{proof}
\begin{proof}[\textbf{Proof for case $q=\infty$, $p \neq \infty$ }]

The value of $\rho^\infty_p(z)_{i,j}$ is the minimal non-negative $\epsilon$ for which the following maximization  problem has non-negative value.

\begin{align}
 \maxop_{x \in \R^d} & \quad \norm{x-w_i}^p_p - \norm{x-w_j}^p_p  \\
	 \textrm{sbj. to:} & \quad \norm{x-z}_\infty \leq \epsilon \nonumber\\
	                   & x \in \X \nonumber
\end{align}

It can be decomposed into $d$ independent problems indexed by $l$.
\begin{align}
 \maxop_{x^{(l)} \in \R} & \quad |x^{(l)}-w^{(l)}_i|^p - |x^{(l)}-w^{(l)}_j|^p  \\
	 \textrm{sbj. to:} & \quad |x^{(l)}-z^{(l)}| \leq \epsilon \nonumber\\
	                   & x^{(l)} \in \X^{(l)} \nonumber
\end{align}
Derivative of the objective function w.r.t.  $x^{(l)}$  is $p|x^{(l)}-w^{(l)}_i|^{p-1}\sign{(x^{(l)}-w_i^{(l)})} - p|x^{(l)}-w^{(l)}_j|^{p-1}\sign{(x^{(l)}-w_j^{(l)})}$, which is non-zero whenever $w_i^{(l)} \neq w_j^{(l)}$. Thus, the maximum is attained at a point where a constraint is active, and the value of the problem is $|z^{(l)}+\epsilon\sign{(w_j^{(l)} - w_i^{(l)})} -w_i^{(l)}|^p - |z^{(l)} + \epsilon\sign{(w_j^{(l)} - w_i^{(l)})} - w_j^{(l)}|^p$.
When $p=2$, the value of the objective is a quadratic function in $\epsilon$; thus, the value of the original objective is also a quadratic function in $\epsilon$ and we can easily obtain a solution to the original problem. For the sake of completeness, we show that this approach results in the same $\rho^\infty_2(z)_{i,j}$ as we derived before:

\begin{equation}
    \begin{aligned}
     \sum_{l=1}^d \left(\left(z^{(l)}+\epsilon\sign{(w_j^{(l)} - w_i^{(l)})} -w_i^{(l)}\right)^2 - \left(z^{(l)} + \epsilon\sign{(w_j^{(l)} - w_i^{(l)})} - w_j^{(l)}\right)^2\right) &\geq 0, \\ 
     \sum_{l=1}^d \left( (z^{(l)} - w_i^{(l)})^2  -(z^{(l)} - w_j^{(l)})^2  + 2\epsilon \sign{(w_j^{(l)} - w_i^{(l)})}(w_j^{(l)} - w_i^{(l)})   \right)
     &\geq 0,\\
    \norm{z-w_i}_2^2 - \norm{z-w_j}_2^2 + 2\epsilon \norm{w_j-w_i}_1 & \geq 0,\\
    \epsilon \geq \frac{\norm{z-w_j}_2^2 - \norm{z-w_i}_2^2}{2\norm{w_j-w_i}_1}.
    \end{aligned}
\end{equation}

If $p=1$, the value of the objective is piecewise linear and non-decreasing; thus, the original objective is again, piecewise linear and non-decreasing. Then we can order the breaking points and find the smallest admissible $\epsilon$ for the original problem using binary search. Note that the objective is maximised not just in the aforementioned case, but also when 
\begin{equation}
  x^{(l)} =\begin{cases}
    w_j^{(l)}, & \text{if $|z^{(l)} - w_j^{(l)}| \leq \epsilon$}. \\
    z_j^{(l)} + \epsilon \sign\left(w_j^{l}-z^{l}\right), & \text{otherwise}.
\end{cases} \end{equation}

For other values of $p$, it may be difficult to solve the problem exactly. However, as we have already shown, it is easy ($\Theta(d)$) to determine if $\rho^\infty_p(z)_{i,j} > \epsilon$ given an $\epsilon$, thus the problem can be solved approximately using binary search for any $p$ with  logarithmic complexity in accuracy.

To conclude the cases $\rho_\infty^p(z)_{i,j}$, we discuss the addition of box constraints. As we have shown, a minimizer of the problems is always $x^* = z + \epsilon\sign{(w_j - w_i)}$, and identical arguments would suggest that with box constraints, it would hold that $x^* = \max(0, \min(1,z + \epsilon\sign{(w_j - w_i)}))$. Therefore, given a radius, certification is done in $O(d)$ even with box constraints. Otherwise, we would need to either order coordinates according to value of $\epsilon$ when a box constraint for $x^* = \max(0, \min(1,z + \epsilon\sign{(w_j - w_i)}))$ becomes active, and then perform a binary search over the constrained problems. This adds a $log(d)$ factor to the complexity. Note that for the case $p=1$ we are already performing a binary search, so we do them at once. Or we can do a binary search over $\epsilon$ to find a minimal one which causes 
\begin{equation}\label{eq:linf_linf_box}
x = \max(0, \min(1,z + \epsilon\sign{(w_j - w_i)}))
\end{equation}
to be misclassified.

\end{proof}
\begin{proof}[\textbf{Proof for case $p=1$, $q\neq\infty$}]
We ave already discussed the case of $\rho_1^\infty(z)_{i,j}$, so it is omitted here. For all other values of $q$, we show its NP-hardness by reducing the knapsack problem to the decision version of problem if given $\epsilon > 0$, $\rho_1^\infty(z)_{i,j} \leq \epsilon$.

\begin{theorem}[Knapsack] The following problem is NP-complete.\ \\
  Given vectors $w, p \in \mathbb{N}^n$ and constants $W,P$. Decide if there is a vector $x \in \{0,1\}^n$ such that $\inner{p,x} \geq P$ and $\inner{w,x} \leq W$.
\end{theorem}

For the sake of clarity, we use $u,v$ instead of $w_i, w_j$ to get rid of unnecessary subscript.
 Let $w,p,W,P$ describe an instance of the knapsack problem.
Let a pair of prototypes $u_t,v_t \in \mathbb{R}^{n+2}$  be defined in the following way for some real $t$ and  $l = 1, \dots, n$
\begin{equation}
  \begin{aligned}
    u^{(l)}_t &= \sqrt[q]{w^{(l)}}, \\
    v^{(l)}_t &= \sqrt[q]{w^{(l)}} - \frac{p^{(l)}}{t},
  \end{aligned}
\end{equation}
let also 
\begin{equation}
  \begin{aligned}
    u^{(n+1)}_{t} &=  \sqrt[q]W + \frac{\max\left(0, \left( 2P -\sum_{i=1}^n p^{(i)}\right)\right)}{t},\\
    v^{(n+1)}_{t} &=  \sqrt[q]W, \\
    u^{(n+2)}_{t} &=  \sqrt[q]W,\\
    v^{(n+2)}_{t} &=  \sqrt[q]W +\frac{\max\left(0, \left(\sum_{i=1}^n p^{(i)}-2P\right)\right)}{t},\\
  \end{aligned}
\end{equation}
and $\epsilon = \sqrt[q]{W}$. 
Now we show that whenever there is an allocation $x \in \{0,1\}^n$ such that $\inner{p,x} \geq P$ and $\inner{w,x} \leq W$, then $\rho_1^q(0) \leq \epsilon$ for any sufficiently large $t$ such that the first $n$ components of $v^{(t)}$ are positive.
It holds that: 
\begin{equation}
  \norm{v_t}_1 \leq \sum_{i=1}^n \sqrt[q]{w^{(i)}} - \sum_{i=1}^n \frac{p^{(i)}}{t} + 2 \sqrt[q]{W} + \frac{\max\left(0, \left(\sum_{i=1}^n p^{(i)}-2P\right)\right)}{t} \leq\sum_{i=1}^n \sqrt[q]{w^{(i)}} + 2\sqrt[q]{W} \leq \norm{u_t}_1.
\end{equation}

Consider the following point
\begin{equation}
  \delta^{(k)} =\begin{cases}
    \sqrt[q]{w^{(k)}}, & \text{if $x^{(k)} = 1$}. \\
    0, & \text{otherwise}.
  \end{cases} \end{equation} 

It has $q$-norm of at most $\epsilon$:
\begin{equation}
  \norm{\delta}_q = \left(\sum_{i=1}^{n+2} \delta^{(i)q}\right)^\frac{1}{q} = \left(\sum_{i=1}^{n} x^{(i)}\cdot w^{(i)}\right)^\frac{1}{q} \leq \sqrt[q]{W} = \epsilon.
\end{equation}

Also it holds that 
\begin{equation}
  \begin{aligned}
  \norm{v_t-\delta}_1 &= \sum_{i=1}^n \left(x^{(i)}\cdot \frac{p^{(i)}}{t} + (1-x^{(i)})\sqrt[q]{w^{(i)}} \right) + 2 \sqrt[q]{W} + \frac{\max\left(0, \left(\sum_{i=1}^n p^{(i)}-2P\right)\right)}{t}, \\
   &\geq \sum_{i=1}^n (1-x^{(i)})\sqrt[q]{w^{(i)}}+ 2 \sqrt[q]{W}  + \frac{P+\max\left(0, \left(\sum_{i=1}^n p^{(i)}-2P\right)\right)}{t},  \\
 &\geq \sum_{i=1}^n (1-x^{(i)})\sqrt[q]{w^{(i)}}  + 2 \sqrt[q]{W} + \frac{\sum_{i=1}^n p^{(i)} - P + \max\left(0, \left( 2P -\sum_{i=1}^n p^{(i)}\right)\right)}{t}\geq \norm{u_t- \delta}_1.
  \end{aligned}
\end{equation}

Therefore, $\rho_1^q(0) \leq \epsilon$.

Now we move on to the second direction; we show that whenever the constructed problem is feasible, then also the knapsack problem is feasible. 

Let there be a $\delta$ such that $\norm{\delta}_q \leq \epsilon$ and $\norm{v_t - \delta} \geq \norm{u_t - \delta}$. Then we can WLoG assume $\delta^{(n+1)} = 0$, and $\sqrt[q]{w^{(l)}} - p^{(l)}/t \leq \delta^{(l)} \leq \sqrt[q]{w^{(l)}}$ for $l = 1, \dots, n$.
Now consider the following allocation for $k = 1, \dots, n$.
\begin{equation}
  x^{(k)} =\begin{cases}
    0, & \text{if $\delta^{(k)} = 0$}. \\
    1, & \text{otherwise}.
  \end{cases} \end{equation}

We show that if $t$ is sufficiently large, then $x$ is a valid allocation. First, let us look at the $\inner{w,x} \leq W$ constraint; 
$$\sum_{i=1}^n \delta^{(i)q} = \sum_{i=1}^n  w^{i}\cdot x^{i} - o(1) = \inner{w,x} - o(1) \leq W;$$
thus, $\inner{w,x} \leq W$ for sufficiently large $t$.
  
For the other constraint, first note for $l = 1, \dots, n$:

\begin{equation}
  \left(|v_t - \delta|^{(l)} - |u_t-\delta|^{(i)}\right) = \begin{cases}
  p^{(i)}/t, & \text{if $x^{(i)} = 0$}. \\
  \geq -p^{(i)}/t , & \text{otherwise}.

  \end{cases}
\end{equation}

Then 

\begin{equation}
\sum_{i=1}^n \left(|v_t - \delta|^{(i)} - |u_t-\delta|^{(i)}\right) \geq \frac{\sum_{i=1}^n p^{(i)} - 2\inner{x,p}}{t},
\end{equation}

and finally

\begin{equation}
  \begin{aligned}
    \frac{\sum_{i=1}^n p^{(i)}  -2P}{t} &\geq \frac{\sum_{i=1}^n p^{(i)} - 2\inner{x,p}}{t},\\
      \inner{x,p} &\geq P.
  \end{aligned}
\end{equation}

\end{proof}

\begin{proof}[\textbf{Proof for case $p=\infty$, $q=1$}]
\begin{align}
 \rho^1_\infty(z)_{i,j}=\minop_{x \in \R^d} & \quad\norm{x-z}_1 \\
	 \textrm{sbj. to:} & \quad \norm{x-w_i}_\infty - \norm{x-w_j}_\infty \geq 0 \nonumber\\
	                   & x \in \X \nonumber
\end{align}

Let $\delta_x = x-z$ where $x \in \arg\min{\rho_\infty^1}(z)_{i,j}$,
We note that there exists $x$ such that $\delta_x$ contains only a single non-zero element. To see why, let there be some $\delta_x$ with multiple non-zero elements from which we construct $\delta_{x'}$ with more zeros such that  $x' \in \arg\min{\rho_\infty^1(z)_{i,j}} $ . Let $l^* = \arg\max_l |x^{(l)} - w_i^{(l)}|$. Take any index $k \neq l^*$ such that $\delta_x^{(k)} \neq 0$. Then consider a perturbation $\delta_x{'}$

\begin{equation}
  \delta_{x'}^{(l)} =  \begin{cases}
  \delta_x^{(l)} + |\delta_x^{(k)}|\sign(x^{(l)} - w_i^{(l)}), & \text{if $l  = l^*$}. \\
  0, & \text{if $l=k$}.\\
  \delta_x^{(l)} , & \text{otherwise}.
  \end{cases}
\end{equation}
Now, $\norm{x' - w_i}_\infty = \norm{x - w_i}_\infty + |\delta_{x}^{k}| \geq \norm{x - w_j}_\infty + |\delta^{(l)}| \geq \norm{\delta'-w_j}_\infty$ which concludes the argument. Now it is sufficient to solve the problem for every coordinate separately and take the maximal value; thus, the original problem is solved in linear time.

\end{proof}

\begin{proof}[\textbf{Proof for case $p=\infty$, $q=2$}]

\begin{align}
 \rho^2_\infty(z)_{i,j}=\minop_{x \in \R^d} & \quad\norm{x-z}_2 \\
	 \textrm{sbj. to:} & \quad \norm{x-w_i}_\infty - \norm{x-w_j}_\infty \geq 0 \nonumber\\
	                   & x \in \X \nonumber
\end{align}

Let $x$ be the minimizer. Then we split the proof into two cases. Either there is an index $l$ such that $\norm{x-w_i}_\infty = |x^{(l)} - w_i^{(l)}| = |x^{(l)} - w_j^{(l)} = \norm{x-w_j}_\infty$. In that case, $|z^{(l)}-w^{(l)}_j| > |z^{(l)}-w^{(l)}_i|$ and $|w^{(l)}_i - w^{(l)}_j|$ is maximal. Then we can compute $x$ in one pass and verify that indeed  $\norm{x-w_i}_\infty =  \norm{x-w_j}_\infty$.

Otherwise, let us Assume that we know $\norm{x-w_i}_\infty = \norm{x-w_j}_\infty= d$ for the optimal $x$. That is, for every coordinate $l$ we have to ensure that $|x^{(l)}- w^{(l)}_j|\leq d$, and also that there is a coordinate $k$ where $|x^{(l)} - w^{(l)}_i| = d$; thus, we can construct $x$ minimizing $\norm{x-z}_2$ as 

\begin{equation}
x^{(l)} = 
\begin{cases}
w_j^{(l)} + d\sign(z^{(l)}-w^{(l)}_i), &\text{if $|w_j^{(l)}-x^{(l)}|>d$}.\\
w^{(l)}_i + d\sign(w_j^{(l)} - w_i^{(l)}), &\text{if $l =k  $}.\\
z^{(l)}, &\text{otherwise},
\end{cases}
\end{equation}
where $k = \min \underset{l}{\arg\max} \quad  \sign\left(w_j^{(l)} -w_i^{(l)}\right)\left( z^{(l)} - w_i^{(l)}  \right)$.

Now we sort (so further we assume the array is sorted) the coordinates according to values of $|w_j^{(l)}-x^{(l)}|$. 

Then minimum of $\norm{x-z}^2_2$ is attained for some $d$ which lies in some interval $[|w_j^{(m)} - x^{(m)}|, |w_j^{(m+1)} - x^{(m+1)}|]$. Inside every such interval, $\norm{x-z}^2_2$ is a quadratic expression in $d$. For the $m$-th interval, the equation is

\[
\norm{x-z}_2^2 = \sum_{l=1}^{m} \left(z^{(l)} - w_j^{(l)} + d\sign(z^{(l)}-w^{(l)}_i)\right)^2 + \left(z^{(l)} -  \sign\left(w_j^{(k)} -w_i^{(k)}\right)\left( z^{(k)} - w_i^{(k)}\right)\right)^2.
\]
So we can minimize a quadratic function $\norm{x-z}_2^2$ over an interval $[|w_j^{(m)} - x^{(m)}|, |w_j^{(m+1)} - x^{(m+1)}|]$. We can also see that for the $m+1$-th equation, we only add one term to the $m$-th equation; thus, we can solve every interval in $O(1)$ and take the minimal $\epsilon$. Consequently, the time complexity is dominated by $O(d\log(d))$ needed for sorting which concludes the proof.

\end{proof}

\section{Proof of Theorem~\ref{thm:rcomplex}}

\textbf{Theorem \ref{thm:rcomplex}}\emph{
The computational complexities of optimization problems $r_p^q(z)_{i,j}$ in \eqref{eq:rqp} for $p,q \in \{1,2,\infty\}$ and $\X=[0,1]^d$ are summarized in Table~\ref{tab:hardness_eps_app}.}
\begin{table}
\begin{center}
\begin{tabular}{c|c| c | c| c|} 
  & & \multicolumn{3}{c|}{$\ell_q$-threat model }\\
 \cline{2-5}
\multirow{4}{*}{ \begin{turn}{90}\hspace{+1mm} $\ell_p$-distance \end{turn}}  & & $\ell_1$ & $\ell_2$ & $\ell_\infty$\\ 
 \cline{2-5}
& $\ell_1$ & NP-hard & NP-hard & Poly \\ 
 \cline{2-5}
 &$\ell_2$ & Poly & Poly & Poly \\
 \cline{2-5}
 &$\ell_\infty$ & NP-hard & NP-hard &NP-hard \\
\cline{2-5}
\end{tabular}
\end{center}
\caption{\label{tab:hardness_eps_app} Computational Complexity of $r^q(z)$.} 
\end{table}
\begin{proof}
The Problem $r_1^q(z)_j$ for $q \neq \infty$ cannot be easier than the problem $\rho_1^q(z)_{i,j}$, thus since the latter is NP-hard, the first also has to be NP-hard.
For the case $r_1^\infty(z)_j$, we recall that the optimal argument of $\rho_1^\infty(z)_{i,j}$ was in the form 
\begin{equation}
  x^{(l)} =\begin{cases}
    w_j^{(l)}, & \text{if $|z^{(l)} - w_j^{(l)}| \leq \epsilon$}, \\
    z_j^{(l)} + \epsilon \sign\left(w_j^{l}-z^{l}\right), & \text{otherwise},
\end{cases} \end{equation}
where $\epsilon$ is the value of $\rho_1^\infty(z)_{i,j}$. Therefore, $r_1^\infty(z)_j = \max_i \rho_1^\infty(z)_{i,j}$ and the overall complexity is $O(d\log(d)|I_c^y|) $. When $p =2$, then the problem reads as
\begin{align*}\label{eq:rqp}
   r_2^q(z)_j=\minop_{x \in \R^d} & \quad\norm{x-z}_q \\
	 \textrm{sbj. to:} & \quad \norm{x-w_i}_2 - \norm{x-w_j}_2 \geq 0 \quad \forall \, i \in I_y \nonumber\\
	                   & x \in [0,1]^d \nonumber
\end{align*}
which is a convex optimization problem for any $q$ and can be solved in polynomial time.

Finally, for the case $p = \infty$ we show that it is $NP-complete$ to solve the feasibility problem of $r_p^q(z)_j$, thus the problem is NP-hard for any $q$. To shorten the notation, we consider $I_y = 1,\dots, n$ and whenever we say that some proposition holds for $w_i$, then we mean it holds for any $w_i$, $i \in 1,\dots,n$.

We show this by reducing $3$-SAT to it. Let there be a formula in CNF $\bigwedge_{i=1}^n \left(\alpha^{(i)} \lor \beta^{(i)} \lor \gamma^{(i)}\right)$, where all all the literals are from a set of $v$ variables. For the sake of brevity, we make a correspondence between the literals and indices $1, \dots, v$. Also when literal corresponding to $i$ is negative, we will write it as $-i$. We will consider the following set of prototypes from $\mathbb{R}^{(v+1)}$.
\[w_j = (0,\dots,0,3) \]
\begin{equation*}
  w_i^{(l)} =\begin{cases}
    -1, & \text{ if $l \in \{\alpha^{(i)},  \beta^{(i)},  \gamma^{(i)}\}  $}, \\
    2, & \text{ if $-l \in \{\alpha^{(i)},  \beta^{(i)},  \gamma^{(i)}\}  $}, \\
    0, & \text{otherwise}.
\end{cases} \end{equation*}
Clearly, for any $x \in [0,1]^d$ it holds that $\norm{w_j-x}_\infty \geq 2$, and also $\norm{x-w_i}_\infty \leq 2$. Therefore, if $r_\infty^p(z)$ is feasible, then $\norm{x-w_i}_\infty = 2$, which is equivalent to proposition $\left(x^{(|\alpha_i|)} = \frac{1+\sign{\alpha_i}}{2}  \right) \lor \left( x^{(|\beta_i|)} = \frac{1+\sign{\beta_i}}{2} \right) \lor \left( x^{(|\gamma_i|)} = \frac{1+\sign{\gamma_i}}{2}\right)$. Such proposition have to be satisfied for every $i$, therefore it is equivalent to a formula in CNF  
\[ \bigwedge_{i=1}^n\left(\left(x^{(|\alpha_i|)} = \frac{1+\sign{\alpha_i}}{2}  \right) \lor \left( x^{(|\beta_i|)} = \frac{1+\sign{\beta_i}}{2} \right) \lor \left( x^{(|\gamma_i|)} = \frac{1+\sign{\gamma_i}}{2}\right)\right), \] 
which is clearly equisatisfiable with the original CNF formula; thus, the feasibility problem is NP-complete.
\end{proof}

\section{Proofs from the Perceptual Metric NPC}\label{app:lpips}
Here we slightly deviate from the main text, that we consider the squared objective which clearly is an equivalent problem 
\begin{align*}\label{eq:l2-opt-lb}
   \vv{\rho^{2}(z)_{i,j}}=\minop_{x \in \R^d} & \quad\norm{x-z}^2_2\\
	 \textrm{sbj. to:} & \quad \inner{x,w_j-w_i} + \frac{\norm{w_i}^2_2-\norm{w_j}^2_2}{2} \geq 0 \nonumber\\
	                   & \quad \norm{x^{(l)}}^2_2=r_l^2 \nonumber \\
	                   & \quad x \geq 0,
\end{align*}
where we use a shortcut $x^{(l)}$, instead of $x^{(h,w,l)}$, to simplify notation.
\begin{proof}[Proof of Proposition \ref{pro:lb-lpips}]
Note that 
\[ \norm{x-z}^2_2 = \sum_{l \in I_l} \norm{x^{(l)}-z^{(l)}}^2_2,\]
and as $\norm{z^{(l)}}_2=r_l$ and we have $\norm{x^{(l)}}_2=r_l$ as constraint, we can equivalently minimize $-\sum_{l \in I_L} \inner{x^{(l)},z^{(l)}}$ as objective.

\vv{Let $v = w_j-w_i$ and $b = \frac{\norm{w_i}^2_2-\norm{w_j}^2_2}{2} $.}
The Lagrangian of the non-convex problem (due to the quadratic \emph{equality} constraints) is
\begin{align*}
    \underset{\mu \geq 0}{L(x,\lambda,\alpha, \mu)} =-\sum_{l \in I_L} \inner{x^{(l)},z^{(l)}}
    + \lambda \Big(\sum_{l \in I_l} \inner{v^{(l)},x^{(l)}}+b\Big)  + \sum_{l \in I_L} \frac{\alpha_l}{2}\Big(\norm{x^{(l)}}^2_2 - r_l^2 \Big) - \sum_{l \in I_L} \inner{\mu^{(l)}, x^{(l)}} 
\end{align*}
We get as critical point condition:
\[ \nabla_{x^{(l)}} L = -z^{(l)} + \lambda v^{(l)} + \alpha_l x^{(l)}-\mu^{(l)}  =0,\]
which yields 
\[ x^{(l)}=\frac{1}{\alpha_l}\left(z^{(l)}-\lambda v^{(l)} + \mu^{(l)} \right).\]
The dual function $q(\lambda,\alpha, \mu)$ becomes
\begin{align*}
 q(\lambda,\alpha, \mu) = 
 &- \sum_{l \in I_L}\frac{1}{\alpha_l} \left(\norm{z^{(l)}}_2^2 -\lambda \inner{v^{(l)}, z^{(l)}} + \inner{\mu^{(l)}, z^{(l)}} \right) 
 \\&+ \lambda \left(\sum_{l\in I_L} \frac{1}{\alpha_l}  \left(\inner{v^{(l)},z^{(l)}} -\lambda \norm{v^{(l)}}_2^2 +\inner{v^{(l)}, \mu^{(l)}}\right) + b \right)\\ &+ \sum_{l\in I_L} \frac{1}{2\alpha_l} \left(\norm{z^{(l)}}_2^2 + \lambda^2\norm{v^{(l)}}_2^2 + \norm{\mu^{(l)}}_2^2 -2\lambda\inner{z^{(l)}, v^{(l)}} + 2\inner{z^{(l)}, \mu^{(l)}} -2\lambda\inner{v^{(l)}, \mu^{(l)}} \right) \\&-\sum_{l\in I_L}\frac{\alpha_l r_l^2}{2}
 -\sum_{l \in I_L} \frac{1}{\alpha} \left( \inner{\mu^{(l)}, z^{(l)}} - \lambda \inner{\mu^{(l)}, v^{(l)}} + \norm{\mu^{(l)}}_2^2\right),
\end{align*}
which simplifies to 
\[ q(\lambda,\alpha, \mu)=-\sum_{l \in I_L}\frac{1}{2\alpha_l}\norm{z^{(l)}-\lambda v^{(l)} + \mu^{(l)}}^2_2 + \lambda b - \sum_{l \in I_L} \frac{\alpha_l r_l^2}{2}.\]
We solve explicitly for $\alpha$ and get 
\[ \alpha_l = \frac{1}{r_l}\norm{z^{(l)} - \lambda v^{(l)} + \mu^{(l)}}_2.\]
Then we get
\[ q(\lambda, \mu)=-\sum_{l \in I_L} \norm{z^{(l)} - \lambda v^{(l)} + \mu^{(l)}}_2 r_l + \lambda b.\]

Solving explicitly for $\mu$, we get 

\[ q(\lambda) = - \sum_{l\in I_L} \norm{ \left(z^{(l)} - \lambda v^{(l)}\right)^+}_2r_l + \lambda b.
\] 
So this is a lower bound on $-\sum_{l \in I_L}\inner{x^{*(l)},z^{(l)}}$, where $x^*$ is the optimal primal variable by weak duality and thus going back to our actual objective we get using $\norm{x^{*(l)}}_2=\norm{z^{(l)}}=r_l$ that
\[ \norm{x^*-z}=\sqrt{\norm{x^*-z}^2_2} = \sqrt{2\sum_{l \in I_L} r_l^2 -2 \sum_{l \in I_L}\inner{x^{*(l)},z^{(l)}}}
\geq \sqrt{2\sum_{l \in I_L} r_l^2 +2 \left(\maxop_{\lambda \geq 0}- \sum_{l\in I_L} \norm{ \left(z^{(l)} - \lambda v^{(l)}\right)^+}_2r_l + \lambda b\right)},\]
where we have used weak duality. Now we go back to indexing using $h,w,l$ instead of just $l$. Since $r_l = \frac{1}{\sqrt{H_lW_l}}$, it holds that 
\[
\sum_{h=1,\dots,H_l}\sum_{w=1,\dots,W_l} r_l^2 =1; 
\]
thus, we can simplify the final expression as 
\[
\sqrt{2L +2 \left(\maxop_{\lambda \geq 0}- \sum_{h,w,l} \norm{ \left(z^{(h,w,l)} - \lambda v^{(h,w,l)}\right)^+}_2r_l + \lambda b\right)}.
\]
Thus we have a one-dimensional convex optimization problem to solve in order to get a lower bound on the original objective, which is all we need for the certification.  
\end{proof}

\section{Results for CIFAR10 with $\ell_2$-NPC}\label{app:cifar}
\textbf{CIFAR10 - $\ell_2$-NPC} In Table~\ref{tab:CIFAR-l2} we compare certified robust accuray (CRA) and an upper bound on the robust accuray (URA) of several models on CIFAR10 for $\ell_2$-threat model. Our $\ell_2$-\ours (800ppc) is slightly better than $\ell_2$-GLVQ (128ppc) in terms of clean accuracy, and robust accuracy for $\epsilon_2 \in \{0.1,36/255\}$, but $\ell_2$-GLVQ is better for $\epsilon_2=0.25$. Note that the $1$-$NN$ is significantly worse showing that learning the prototypes helps improving the performance. Nevertheless, all NPC models are not competitive with neural networks which is to be expected as the $\ell_2$-distance is not a good measure for image similarity. This is why we study \ours with the perceptual metric which achieves to clean accuracies which are higher than the one of neural networks with provable robustness guarantees.

Table~\ref{tab:CIFAR10-mult} shows the performance of $\ell_2$-NPC for multiple threat models. $\ell_2$-\ours outperforms $\ell_2$-GLVQ in terms of clean accuracy, $\ell_1$- and $\ell_2$-robust accuracy but is worse for $\ell_\infty$-robust accuracy and as this is the most difficult threat model it is also worse in the union. MMR-U outperforms the $\ell_2$-NPC but the margin is relatively small. 

\begin{table}[t]
\caption{\textbf{CIFAR10:} lower (CRA) and upper bounds (URA) on $\ell_2$-robust accuracy}\label{tab:CIFAR-l2}
\begin{center}
\begin{small}
\setlength{\tabcolsep}{2pt}
\begin{tabular}{l|c|cc|cc|cc|}
CIFAR10 & std. &  
\multicolumn{2}{c|}{$\epsilon_2=0.1$} & 
\multicolumn{2}{c|}{$\epsilon_2=36/255$}
& \multicolumn{2}{c|}{$\epsilon_2=0.25$}\\
& acc. & CRA & URA & CRA & URA & CRA & URA\\
\midrule
\ours & 49.2 & 43.9 & 43.9 & 41.9 & 41.9 & 36.4 & 36.4\\
GLVQ 
& 48.6 & 43.3  & 43.3 & 41.5 & 41.5 & 37.9 & 37.9\\
1-NN 
& 35.7 & 31.2 & - & 29.7 & 29.7 & 25.7 & -\\
\midrule
GloRob 
& 77.0 & - & - & 58.4 & 69.2 & - & - \\
LocLip 
& \textbf{77.4} & - & - & \textbf{60.7} & 70.4 & - & - \\
BCP 
& 65.7 & - & - & 51.3 & 60.8 & - & -\\
\midrule
SmoothLip$_{\sigma=0.25}$ 
&   77.1 & - & - & - & -  & \textbf{67.9$^*$}  & 67.9$^*$\\
\bottomrule
\end{tabular}
\end{small}
\end{center}
\vskip -0.1in
\end{table}

\begin{table}[t]
\caption{\textbf{CIFAR10:} lower (CRA) and upper bounds (URA) on robust accuracy for multiple threat models for our $\ell_2$-\ours, the $\ell_2$-NPC of \cite{Saralajew2020fast}, a 1-NN classifier. As comparison we show MMR-Univ of \cite{croce2020provable} which is a neural network specifically trained for certifiable multiple-norm robustness.\label{tab:CIFAR10-mult}}
\begin{center}
\begin{small}
\setlength{\tabcolsep}{2pt}
\begin{tabular}{l|c|cc|cc|cc|cc|}
CIFAR10 & std. & \multicolumn{2}{c|}{$\epsilon_1=2$} & 
\multicolumn{2}{c|}{$\epsilon_2=0.1$} & 
\multicolumn{2}{c|}{$\epsilon_\infty=2/255$} &\multicolumn{2}{c|}{union}\\
& acc. & CRA & URA & CRA & URA & CRA & URA & CRA & URA\\
\midrule
$\ell_2$-\ours & 49.2 & \textbf{42.5}  & 42.5 & 41.9 & 41.9 & 32.7 & 32.7& 32.7  & 32.7\\
$\ell_2$-GLVQ & 48.6 & 42.3 & 42.3 & 41.5 & 41.5 & 35.2 & 35.2 &  35.2 & 35.2 \\
1-NN  & 35.7 & 30.0 & - & 29.7 & 29.7 & 22.5 & - & 22.5 & -\\
\midrule
MMR-U & \textbf{53.0} & 36.6 & 43.6 & \textbf{46.4} & 48.1 & \textbf{36.2} & 36.2 & \textbf{35.4} & 36.2\\
\bottomrule
\end{tabular}
\end{small}
\end{center}
\vskip -0.1in
\end{table}

\newpage
\section{Comparison with orthogonal convolution networks}\label{app:SOC}

We evaluated the robustness orthogonal convolution networks on MNIST at radius $1.58$. According to the evaluation in~\cite{singla2022improved}, the currently best method for orthogonal convolution networks is to combine skew orthogonal convolutions with Householder activations. According to the official repository, they suggest to choose to set the following parameters
\begin{itemize}
    \item \texttt{--conv-layer} - We chose \texttt{soc} because it consistently outperformed baselines in the paper.
    \item \texttt{--activation} - We chose \texttt{hh1} activation, which is used in the experiments in the original paper.
    \item \texttt{--num-blocks} - We tried $1,2,4,6$ blocks, possible values are $1\dots8$. In the original paper, it did not seem that more blocks boost performance.
    \item \texttt{--gamma} - We tried $0, 0.1, 0.2, 0.5, 1$. The original experiments used $0.1$.
    \item \texttt{--lln} - The authors suggest to use last layer normalization when the number of classes is large, e.g., for CIFAR100, and do not use it for CIFAR10. We also did not use it.
\end{itemize}

We padded the MNIST images by $2$ black pixels, so that we can directly use the original architecture which relied on the fact that the input images are $32\times 32$. We also turned off the normalization by mean and variance as it is not commonly use for MNIST. We removed random horizontal flip from the set of possible augmentations, otherwise the setup is exactly as recommended. We note that the padding of MNIST image by $2$ pixels is likely not the optimal way how to adapt the network to work with MNIST dataset.

The orthogonal convolutions from~\cite{li2019preventing} reports $56.4\%$ certified robust accuracy. The method of~\cite{trockman2021orthogonalizing} yielded $54\%$ robust accuracy with the suggested setup.

\subsection{Empirical robustness}
We evaluated the empirical robustness of~\cite{singla2022improved} using AutoAttack which is a stronger attack than what the competing methods used in Table~\ref{tab:MNIST-l2}. Thus, we don't conclude that orthogonal convolutions are (significantly) less empirically robust than the other evaluated methods.

\begin{table}[t]
\caption{\textbf{MNIST:} Certified robust accuracy of networks with orthogonal convolutions. We computed robust accuracy after every epoch and the reported numbers are the maximal ones. The radius is $1.58$
\label{tab:MNIST-conv}}
\begin{center}
\begin{tabular}{|l||*{5}{c|}}\hline
\backslashbox{blocks}{$\gamma$}
&\makebox[3em]{$0$}&\makebox[3em]{$0.1$}&\makebox[3em]{$0.2$}
&\makebox[3em]{$0.5$}&\makebox[3em]{$1$}\\\hline\hline
$1$&57.17&58.23&58.75&58.82&58.57\\\hline
$2$&58.31&58.85&59.63&59.21&58.99\\\hline
$4$&59.50&60.75&\textbf{61.02}&60.33&58.82\\\hline
$6$&59.78&60.47&59.05&59.99&57.53\\\hline
\end{tabular}
\end{center}
\end{table}

\end{document}